\documentclass{article}

\usepackage{iclr2026_conference,times}
\usepackage{threeparttable}
\usepackage{siunitx}
\usepackage{adjustbox}
\usepackage{microtype}
\usepackage{amsmath,amssymb,amsfonts}
\usepackage{amsthm}
\usepackage{aliascnt}
\usepackage{booktabs}
\usepackage{graphicx}
\usepackage{caption}
\usepackage{float}
\usepackage{algorithm}
\usepackage{algpseudocode}
\usepackage{multirow}
\usepackage{natbib}
\usepackage{xcolor}
\usepackage{url}
\usepackage{hyperref}
\usepackage{enumitem}
\usepackage{longtable}
\usepackage{tikz}
\usetikzlibrary{arrows.meta,positioning,calc,fit}
\usepackage{pgfplots}
\usepgfplotslibrary{groupplots}
\pgfplotsset{compat=1.18}

\hypersetup{
  colorlinks=true,
  citecolor=green!55!black,
  linkcolor=green!55!black,
  urlcolor=green!55!black
}

\newcommand{\method}{\textsc{PDC}}
\newcommand{\init}{\textsc{Init}}
\newcommand{\tif}{\textsc{TiF}}

\newcommand{\masktoken}{\text{\texttt{[MASK]}}}
\newcommand{\token}{\text{\texttt{[TOKEN]}}}
\newcommand{\posgain}[1]{\textcolor{green!45!black}{#1}}
\newcommand{\neggain}[1]{\textcolor{red!55!black}{#1}}

\newcommand{\Acc}{\operatorname{Acc}}
\newcommand{\Maj}{\operatorname{Maj}}
\theoremstyle{plain}
\newtheorem{theorem}{Theorem}[section]
\newaliascnt{proposition}{theorem}
\newtheorem{proposition}[proposition]{Proposition}
\aliascntresetthe{proposition}
\newaliascnt{corollary}{theorem}
\newtheorem{corollary}[corollary]{Corollary}
\aliascntresetthe{corollary}
\newaliascnt{lemma}{theorem}
\newtheorem{lemma}[lemma]{Lemma}
\aliascntresetthe{lemma}

\theoremstyle{definition}
\newaliascnt{definition}{theorem}

\aliascntresetthe{definition}
\newaliascnt{assumption}{theorem}
\newtheorem{assumption}[assumption]{Assumption}
\aliascntresetthe{assumption}

\newaliascnt{observation}{theorem}
\newtheorem{observation}[observation]{Observation}
\aliascntresetthe{observation}

\theoremstyle{remark}
\newaliascnt{remark}{theorem}

\aliascntresetthe{remark}
\usepackage{listings}
\DeclareCaptionStyle{ruled}{labelfont=normalfont,labelsep=colon,strut=off}
\floatstyle{ruled}
\newfloat{listing}{tb}{lst}{}
\floatname{listing}{Algorithm}

\usepackage{booktabs}
\usepackage{multirow}
\usepackage{makecell}
\usepackage{tcolorbox}
\tcbuselibrary{breakable}
\usepackage[framemethod=tikz]{mdframed}

\newenvironment{mainobservation}[1][]{%
  \begin{tcolorbox}[
    colback = green!5,
    colframe = white,
    fonttitle = \bfseries,
    breakable = true]
  \begin{observation}[#1]%
}{%
  \end{observation}%
  \end{tcolorbox}
  \vspace{5pt}
}

\newenvironment{maintheorem}[1][]{%
  \begin{tcolorbox}[
    colback = red!5,
    colframe = white,
    fonttitle = \bfseries,
    breakable = true]
  \begin{theorem}[#1]%
}{%
  \end{theorem}%
  \end{tcolorbox}
  \vspace{5pt}
}

\title{Prefix-Denoising Consistency: Test-Time Verification for Diffusion Language Models}

\author{%
  Yuki Ichihara\textsuperscript{1} \quad
  Naoto Iwase\textsuperscript{2}  \quad
  Mohammad Atif Quamar\textsuperscript{1}\quad
  Junpei Komiyama\textsuperscript{1,3} \\[2pt]
  \textsuperscript{1}MBZUAI \quad
  \textsuperscript{2}Nagoya University \quad
  \textsuperscript{3}RIKEN AIP \\[2pt]
  \normalfont\small\texttt{\{yuki.ichihara, mohammad.atif\}@mbzuai.ac.ae} \\ \texttt{naoto@iwase.dev} \quad \texttt{junpei@komiyama.info} \\[2pt]
}
\date{}
\iclrfinalcopy
\begin{document}

\maketitle
\lhead{Preprint.}

\begin{abstract}
Diffusion Language Models (DLMs) have recently become increasingly competitive with autoregressive (AR) models, and even outperform them on certain tasks. Unlike AR models, DLMs produce output through iterative denoising without a left-to-right order. To further improve the performance of DLMs, we introduce \method{} (\emph{Prefix-Denoising Consistency}), a test-time self-verification method for DLMs. \method{} exploits a distinctive test-time signal in DLMs under prefix conditioned regeneration, correct trajectories are more stable and reproducible than incorrect ones. Concretely, given an initially generated sample, \method{} splits the sentence at an intermediate position and regenerates the remaining tokens conditioned on the fixed prefix. Across mathematical reasoning and commonsense reasoning benchmarks, \method{} consistently improves upon the initial sample, outperforms independent generations under a computational constrained comparison, and is robust to different unmasking strategies and parameter settings. These results highlight prefix-conditioned regeneration as an effective DLM-specific primitive for test-time verification.

\end{abstract}

\begin{figure}[h]
\centering
\resizebox{0.98\linewidth}{!}{%
\begin{tikzpicture}[
  font=\footnotesize,
  >=Latex,
  panel/.style={draw=blue!35!black, fill=blue!2, rounded corners=4pt,
    line width=0.55pt},
  methodtitle/.style={font=\bfseries\large, text=black!84, align=center},
  heading/.style={font=\bfseries, text=black!84, align=center},
  subnote/.style={font=\scriptsize, text=black!62, align=center},
  arrowlabel/.style={font=\scriptsize, text=black!58, fill=white,
    inner xsep=2pt, inner ysep=1pt, align=center},
  rate/.style={font=\scriptsize, text=black!68, align=right},
  box/.style={draw=black!30, fill=white, rounded corners=2pt, line width=0.45pt,
    align=center, inner xsep=5pt, inner ysep=4pt},
  answer/.style={box, fill=green!10, draw=green!55!black, minimum width=0.86cm,
    inner xsep=3pt, inner ysep=3pt},
vote/.style={box, fill=green!10, draw=green!55!black,
  text width=2.05cm,
  minimum height=1.12cm,
  inner xsep=5pt,
  inner ysep=5pt},
  arrow/.style={-{Latex[length=2.0mm,width=1.35mm]}, line width=0.62pt,
    draw=black!60, shorten >=2pt, shorten <=2pt},
  dashedarrow/.style={-{Latex[length=1.8mm,width=1.20mm]}, line width=0.48pt,
    draw=black!38, dashed, shorten >=2pt, shorten <=2pt},
  barframe/.style={draw=black!50, line width=0.34pt},
  keptseg/.style={fill=red!16, draw=black!34, line width=0.24pt},
  maskseg/.style={fill=black!8, draw=black!30, line width=0.24pt},
  regenseg/.style={fill=red!16, draw=black!32, line width=0.24pt},
  initseg/.style={fill=orange!18, draw=black!34, line width=0.24pt}
]
\draw[panel] (-1.55,-1.02) rectangle (13.05,3.12);
\node[methodtitle] at (5.75,2.52) {Prefix-Denoising Consistency};

\def\totw{3.20}
\def\barh{0.28}
\def\initialbar#1#2{%
  \begin{scope}[shift={(#1,#2)}]
    \draw[initseg] (0,0) rectangle (\totw,\barh);
    \draw[barframe] (0,0) rectangle (\totw,\barh);
    \node[font=\scriptsize,text=black!72] at (1.60,0.14) {$y$};
  \end{scope}%
}
\def\inputbar#1#2#3#4{%
  \begin{scope}[shift={(#1,#2)}]
    \node[rate, anchor=east] at (-0.28,0.14) {#4};
    \draw[keptseg] (0,0) rectangle (#3,\barh);
    \draw[maskseg] (#3,0) rectangle (\totw,\barh);
    \draw[barframe] (0,0) rectangle (\totw,\barh);
  \end{scope}%
}
\def\outputbar#1#2#3#4{%
  \begin{scope}[shift={(#1,#2)}]
    \draw[keptseg] (0,0) rectangle (#3,\barh);
    \draw[regenseg] (#3,0) rectangle (\totw,\barh);
    \draw[barframe] (0,0) rectangle (\totw,\barh);
    \node[answer] at (3.76,0.14) {#4};
  \end{scope}%
}

\node[heading] at (1.60,1.78) {Prefix conditioned inputs};
\node[heading] at (7.05,1.78) {Regenerated outputs};
\node[heading] at (11.35,1.78) {Vote};

\inputbar{0.00}{1.18}{0.32}{$\rho=0.1$}
\inputbar{0.00}{0.56}{1.60}{$\rho=0.5$}
\inputbar{0.00}{-0.06}{2.88}{$\rho=0.9$}

\outputbar{5.35}{1.18}{0.32}{$a^{(0.1)}$}
\outputbar{5.35}{0.56}{1.60}{$a^{(0.5)}$}
\outputbar{5.35}{-0.06}{2.88}{$a^{(0.9)}$}
\draw[arrow] (3.20,1.32) -- (5.35,1.32);
\draw[arrow] (3.20,0.70) -- (5.35,0.70);
\draw[arrow] (3.20,0.08) -- (5.35,0.08);
\node[arrowlabel] at (4.28,0.96) {DLM regeneration};

\node[vote] (vote) at (11.35,0.7)
{majority voting\\[1.5pt]\hspace*{-0.05cm}%
${\scriptstyle a^{(0.1)}, \,a^{(0.5)}, \,a^{(0.9)}}$};
\coordinate (votein10) at ($(vote.west)+(0,0.62)$);
\coordinate (votein50) at (vote.west);
\coordinate (votein90) at ($(vote.west)+(0,-0.62)$);
\draw[arrow] (9.54,1.32) -- (votein10);
\draw[arrow] (9.54,0.70) -- (votein50);
\draw[arrow] (9.54,0.08) -- (votein90);

\node[anchor=west,font=\scriptsize,text=black!62] at (0.00,-0.58)
  {\tikz{\draw[keptseg] (0,0) rectangle (0.32,0.15);} unmasked positions \quad
   \tikz{\draw[maskseg] (0,0) rectangle (0.32,0.15);} masked positions};
\end{tikzpicture}%
}
\caption{
Overview of Prefix-Denoising Consistency (\method{}). Given an initial Diffusion Language Model output, \method{} keeps exact
prefixes at rates $\rho \in \{0.1, 0.5, 0.9\}$, remasks the
remaining positions within the same length-$L$ output, and
regenerates them to obtain candidate answers. The final answer is
selected by majority vote over the regenerated answers only; the initial
answer is recorded for analysis but excluded from the vote.
}
\label{fig:regen-overview}
\end{figure}
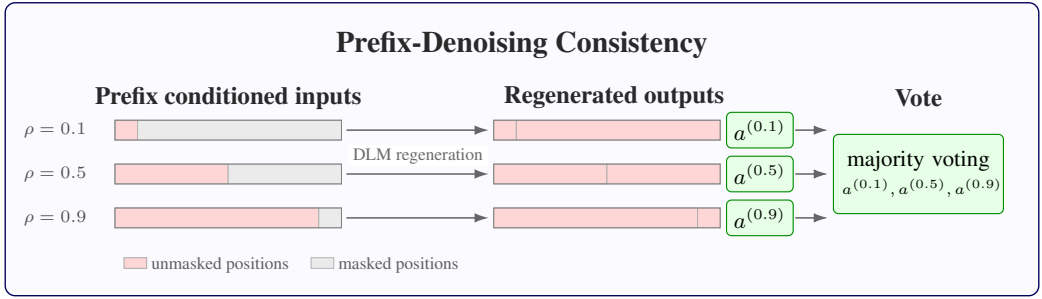

\section{Introduction}
Diffusion Language Models (DLMs)
\citep{nie2025llada,zhu2025llada15,ye2025dream} have recently emerged as
a compelling alternative to autoregressive (AR) Language Models
\citep{openai2025gptoss,yang2025qwen3,nvidia2025nemotron3nano}. Unlike AR
models, which generate tokens strictly from left to right, DLMs refine
outputs through iterative denoising, enabling more global revision and
potentially more parallel and efficient inference. Recent DLMs have shown
competitive performance across a range of language and reasoning tasks \citep{gong2025scaling,fu2026nemotronlabsdiffusion},
suggesting that diffusion-based generation can offer a promising new
scaling direction beyond the standard AR paradigm.

Compared to AR models, state-of-the-art DLMs are decoded at low temperature to obtain
strong pass@1 accuracy. However, this also creates a challenge
for test-time verification: due to the low temperature, repeated samples often have limited diversity,
because the DLM can follow similar denoising trajectories and return
the same answer. This makes a naive transfer of self-consistency from AR models less direct. In AR reasoning, self-consistency improves
accuracy by sampling multiple reasoning paths and selecting the most
frequent answer \citep{wang2023selfconsistency}. For low-temperature
DLM decoding, however, repeated full generations can result in
near-identical denoising trajectories, so agreement among samples may
overstate reliability rather than provide an independent check.
Recent work \citep{wang2026time} addresses this issue by using temporal
information (\tif{}) inside the unmasking process. \tif{} votes over answers
extracted from the information of intermediate denoising steps,
showing that the denoising trajectory contains a useful reasoning signal
and can improve accuracy. However, \tif{} exploits consistency within a
single denoising trajectory: it reuses intermediate states from one run
rather than generating multiple alternative reasoning trajectories. By
construction, it therefore cannot directly leverage the multi-trajectory
self-consistency mechanism that makes majority voting effective in AR
reasoning.
\tif{} can exploit temporal fluctuations along that run, but it cannot branch into a new trajectory once the run becomes locked into an erroneous reasoning path.

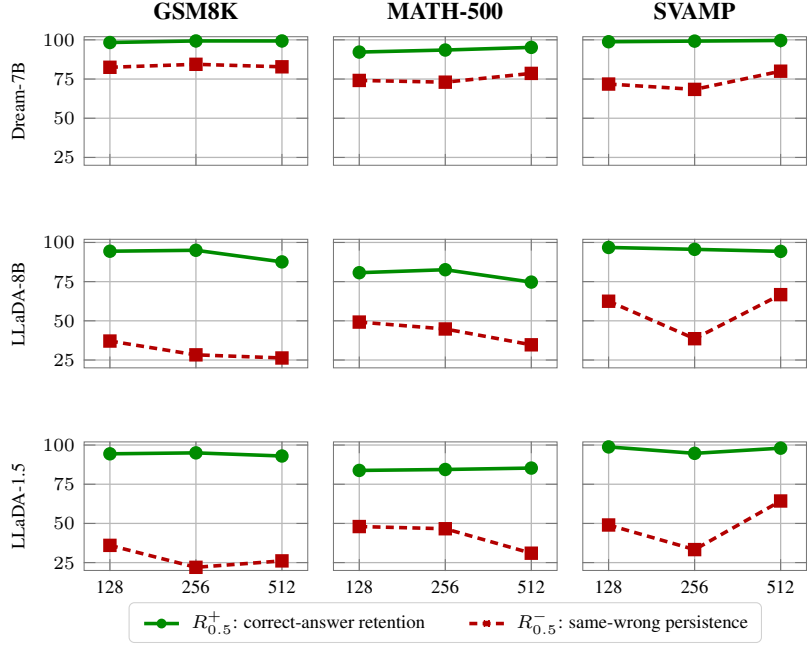
\begin{figure*}[t]
\centering
\begin{tikzpicture}
\begin{groupplot}[
  group style={
    group name=stability,
    group size=3 by 3,
    horizontal sep=0.34cm,
    vertical sep=0.98cm
  },
  width=0.325\textwidth,
  height=0.235\textwidth,
  xmin=0.7, xmax=3.3,
  ymin=20, ymax=102,
  xtick={1,2,3},
  xticklabels={,,},
  ytick={25,50,75,100},
  tick label style={font=\scriptsize},
  title style={font=\footnotesize\bfseries, yshift=-1mm},
  label style={font=\scriptsize},
  grid=major,
  grid style={draw=black!28, line width=0.48pt, line cap=rect},
  axis line style={draw=black!55, line width=0.35pt},
  tick style={draw=black!55, line width=0.35pt},
  every axis plot/.append style={line width=1.35pt, mark size=1.85pt},
  rplus/.style={green!55!black, solid, mark=*,
    mark options={solid, fill=green!55!black, draw=green!55!black}},
  rminus/.style={red!70!black, dash pattern=on 3pt off 2pt, mark=square*,
    mark options={solid, fill=red!70!black, draw=red!70!black}},
]

\nextgroupplot[title={GSM8K}, ylabel={Dream-7B}]
\addplot[rplus] coordinates {(1,98.3) (2,99.3) (3,99.3)};
\addplot[rminus] coordinates {(1,82.5) (2,84.4) (3,82.8)};
\nextgroupplot[title={MATH-500}, yticklabels={}]
\addplot[rplus] coordinates {(1,92.2) (2,93.5) (3,95.2)};
\addplot[rminus] coordinates {(1,74.1) (2,73.0) (3,78.6)};
\nextgroupplot[title={SVAMP}, yticklabels={}]
\addplot[rplus] coordinates {(1,98.8) (2,99.2) (3,99.6)};
\addplot[rminus] coordinates {(1,71.8) (2,68.4) (3,80.0)};

\nextgroupplot[ylabel={LLaDA-8B}]
\addplot[rplus] coordinates {(1,94.4) (2,95.0) (3,87.6)};
\addplot[rminus] coordinates {(1,37.1) (2,28.3) (3,26.3)};
\nextgroupplot[yticklabels={}]
\addplot[rplus] coordinates {(1,80.7) (2,82.6) (3,74.7)};
\addplot[rminus] coordinates {(1,49.2) (2,44.8) (3,34.7)};
\nextgroupplot[yticklabels={}]
\addplot[rplus] coordinates {(1,96.8) (2,95.6) (3,94.3)};
\addplot[rminus] coordinates {(1,62.5) (2,38.6) (3,66.7)};

\nextgroupplot[ylabel={LLaDA-1.5}, xticklabels={128,256,512}]
\addplot[rplus] coordinates {(1,94.4) (2,95.0) (3,93.0)};
\addplot[rminus] coordinates {(1,36.0) (2,22.0) (3,26.1)};
\nextgroupplot[xticklabels={128,256,512}, yticklabels={}]
\addplot[rplus] coordinates {(1,83.8) (2,84.4) (3,85.3)};
\addplot[rminus] coordinates {(1,48.0) (2,46.6) (3,31.0)};
\nextgroupplot[xticklabels={128,256,512}, yticklabels={}]
\addplot[rplus] coordinates {(1,98.8) (2,94.7) (3,98.0)};
\addplot[rminus] coordinates {(1,49.0) (2,33.3) (3,64.3)};

\end{groupplot}
\node[
  anchor=north,
  font=\scriptsize,
  fill=white,
  draw=black!18,
  rounded corners=2pt,
  inner xsep=5pt,
  inner ysep=3pt
] at ($(stability c2r3.south)+(0,-0.42cm)$) {%
  \begin{tabular}{@{}c@{\hspace{4pt}}l@{\hspace{16pt}}c@{\hspace{4pt}}l@{}}
  \tikz[baseline=-0.55ex]{\draw[green!55!black,solid,line width=1.35pt] (0,0)--(0.46,0); \fill[green!55!black] (0.23,0) circle[radius=1.85pt];}
  & $R^+_{0.5}$: correct-answer retention
  & \tikz[baseline=-0.55ex]{\draw[red!70!black,dash pattern=on 3pt off 2pt,line width=1.35pt] (0,0)--(0.46,0); \filldraw[red!70!black] (0.20,-0.03) rectangle (0.26,0.03);}
  & $R^-_{0.5}$: same-wrong persistence
  \end{tabular}%
};
\end{tikzpicture}
\caption{
Reproduction rates under $\rho= 0.5$ regeneration.
Green denotes correct-answer preservation $R^+_{0.5}$, while red denotes same-wrong-answer preservation $R^-_{0.5}$, across generation lengths $L=128,256,512$ on three math benchmarks.
The gap $R^+_{0.5} - R^-_{0.5} > 0$ shows that correct answers are more reproducible than wrong answers. Full results of reproduction rates, including CSQA and SQA, are reported in
Appendix Table~\ref{tab:regen-consistency-all}.
}
\label{fig:regen-consistency-rho50}
\end{figure*}
To address this issue, we propose \method{} (\emph{Prefix-Denoising Consistency}). Starting from a completed output, it keeps an exact prefix, remasks the remaining positions, and re-denoises them at the same temperature. Repeating this intervention at several keep rates (i.e., amount of the prefix kept) produces structured alternative trajectories, whose extracted answers are
aggregated by majority vote. This allows the model to revise errors in the original generation while preserving useful context from the initial solution.

Our contributions are:
\begin{itemize}[leftmargin=1.4em,itemsep=2pt,topsep=3pt]
    \item \textbf{Prefix-conditioned regeneration.} \method{} is a
    test-time scaling method that holds exact
    prefixes inside a fixed output window, remasks the unkept positions,
    and votes only over regenerated answers. The initial answer is
    excluded to isolate the effect of prefix-conditioned denoising.
    \item \textbf{Initially correct answers are more reproducible. 
    }
    Figure~\ref{fig:regen-consistency-rho50} shows that initially correct
    answers are more reproducible than initially wrong answers. This observation leads to our algorithm. 
    \item \textbf{Improved performance over the initial sample. 
    } In the full diffusion
    setting across math and commonsense benchmarks, \method{} consistently improves
    over the initial sample and \tif{}.
    \item \textbf{Compute efficiency and robustness analysis.} 
    We compare \method{} to standard self-consistency (i.e., majority voting). We show that PDC outperforms majority voting with a smaller total denoising-step budget. We also conduct validation under different unmasking schemes and settings and observe improvements.
\end{itemize}

\section{Prefix-Denoising Consistency}

\subsection{Standard DLM Inference Time}

DLMs generate outputs by iteratively denoising a fixed-length sequence initialized with mask tokens. Given a prompt $x$, a maximum output length $L$, and a total number of denoising steps $T$, let $\mathcal{V}^{\leq L}$ denote the set of token sequences over the vocabulary $\mathcal{V}$ with length at most $L$.

The denoising process begins with an output window consisting entirely of mask tokens:
\begin{align}\label{eq:z}
    z^{(0)}
    &=
    \left[
    x,\;
    \underbrace{\masktoken,\ldots,\masktoken}_{L}
    \right],
\end{align}

At each denoising step $t=1,\ldots,T$, the DLM predicts tokens for all currently masked output positions in parallel, conditioned on the partially denoised sequence $z^{(t-1)}$. It then selects a subset of these positions to unmask according to a denoising schedule, often based on the model's confidence. The remaining positions stay masked, producing an updated sequence $z^{(t)}$, for example,
\begin{equation*}
    z^{(t)}
    =
    \left[
    x,\;
    \underbrace{
    \masktoken,\token,\ldots,
    \token,\masktoken,\ldots
    }_{L}
    \right].
\end{equation*}
Thus, at an intermediate denoising step, the output window contains both positions that remain masked and tokens filled in during earlier steps. The positions are not necessarily unmasked in left-to-right order; instead, they may be resolved in an arbitrary order determined by the denoising schedule.
After $T$ denoising steps, DLM obtains a fully unmasked output window, output sequence $y = z^{(T)}$, whose length is at most the output length $L$. 
\begin{algorithm}[t]
\caption{\textbf{Prefix-Denoising Consistency (\method{})}}
\label{alg:regen}
\begin{algorithmic}[1]
\Require Prompt $x$, generation budget $L$, denoising budget $T$, keep rates $\mathcal{R}=\{0.1,0.5,0.9\}$, answer extractor $\mathrm{Extract}$ 
\Ensure Regeneration-vote answer $\hat a$, or no-vote
\State $y \gets \mathrm{DLM}(x;L,T)$; $\ell_y \gets |y|$
\State Record $a^{(0)} \gets \mathrm{Extract}(y)$ for diagnostics only
\For{each keep rate $\rho \in \mathcal R$}
    \If{$\ell_y=0$}
        \State $L_\rho^{\text{prefix}} \gets 0$
    \Else
        \State $L_\rho^{\text{prefix}} \gets \max(1,\lfloor \rho \ell_y \rfloor)$
    \EndIf
    \State $L_\rho \gets L-L_\rho^{\text{prefix}}$
    \State $T_\rho \gets L_\rho$
    \State Run the constrained DLM with output positions $1{:}L_\rho^{\text{prefix}}$ fixed and $L_\rho$ suffix positions masked:
    \Statex \hspace{\algorithmicindent}
    $y^{(\rho)}
    \gets
    \mathrm{DLM}
    \left(x,y_{1:L_\rho^{\text{prefix}}};L_\rho,T_\rho\right)$
    \State $a^{(\rho)} \gets \mathrm{Extract}(y^{(\rho)})$
\EndFor
\State $\mathcal B \gets
\{ a^{(\rho)}:\rho\in\mathcal R,\;a^{(\rho)}\neq\varnothing\}$
\If{$\mathcal B=\{\}$}
    \State \Return no-vote
\EndIf
\State \Return $\Maj(\mathcal B)$, breaking ties by the fixed order
$\rho=0.1,0.5,0.9$
\Statex \emph{Note:} the initial answer $a^{(0)}$ is excluded from the vote.
\end{algorithmic}
\end{algorithm}

Our method does not rely on the internal form of the DLM update rule.
We therefore treat the DLM as a black-box generator that, given an input
$x$, an output budget $L$, and a number of denoising steps $T$, returns
a sequence
\begin{equation*}
    y
    =
    \mathrm{DLM}(x; L, T),
    \qquad
    y \in \mathcal{V}^{\leq L}.
    \label{eq:black-box-dlm}
\end{equation*}

The output length and initial extracted answer are:
\begin{equation*}
    \ell_y = |y|,
    \qquad
    a^{(0)} = \mathrm{Extract}(y),
\end{equation*}
where $\mathrm{Extract}$ is a task-specific canonical answer extractor
and $\mathrm{Extract}(y)=\varnothing$ denotes a null extraction that represents a failure to parse the answer.

In a typical diffusion decoding, the model predicts tokens for the currently masked positions in
parallel, and a subset of positions is unmasked at each step according to
a denoising schedule, often based on confidence \citep{nie2025llada,zhu2025llada15,ye2025dream}.

\subsection{Prefix-Denoising Consistency (\method{})}\label{sec:pdc}
Our proposed method, \method{} (Figure~\ref{fig:regen-overview} and Algorithm~\ref{alg:regen}) holds a prefix of the initial generated output $y$ at a keep rate $\rho\in\mathcal R$, masks the remaining output positions, regenerates those positions using the same DLM, and then aggregates the regenerated answers by majority vote.

Throughout the paper, we use a fixed set of three rates $\mathcal R=\{0.1,0.5,0.9\}$, this set covers weak, intermediate, and strong prefix conditioning while requiring only three regenerations.
Given a keep initial output $y$ of length $\ell_y$, the number of held tokens for rate $\rho$ is:
\begin{equation*}
    L_\rho^{\text{prefix}} =
         \lfloor \rho \ell_y\rfloor.
    \label{eq:keep-length}
\end{equation*}
The remaining regeneration length and denoising step are:
\begin{equation*}
    L_\rho = L-L_\rho^{\text{prefix}},
    \qquad
    T_\rho = L_\rho.
    \label{eq:regen-steps}
\end{equation*}
Different from Eq.~\eqref{eq:z}, the regeneration input is the following equation:
\begin{equation*}
    z^{(0)}_\rho = \left[
    x,\;
    y_{1:L_\rho^{\text{prefix}}},\;
    \underbrace{\masktoken,\ldots,\masktoken}_{L_\rho}
    \right],
\end{equation*}
where $y_{1:L_\rho^{\text{prefix}}}$ is kept output tokens with a rate of $\rho$ and $L_\rho$ is the number of masked suffix positions to denoise.
Only the remaining $L_\rho$ output positions are denoised:
\begin{equation*}
    y^{(\rho)}
    =
    \mathrm{DLM}
    \left(x,y_{1:L_\rho^{\text{prefix}}};L_\rho,T_\rho\right),
    \qquad
    |y^{(\rho)}|\leq L.
    \label{eq:constrained-dlm}
\end{equation*}
The regenerated answer is:
\begin{equation*}
    a^{(\rho)} = \mathrm{Extract}(y^{(\rho)}).
\end{equation*}
Thus, the final length of $y^{(\rho)}$ may differ from the initial output length $\ell_y$. However, regeneration does not continue from the end of the initial output. Instead, it only refills the remasked positions within the same output window of length $L$.

Motivated by analogous observations for AR models \citep{iwase2026reliable}, we find that the correct reasoning paths are more reproducible under regeneration than incorrect ones in DLMs (Figure ~\ref{fig:regen-consistency-rho50}).
The gold answer is denoted by $a^\star$. Reproduction of the correct answer 
measures whether regeneration preserves a correct initial answer, is defined as:
\begin{equation*}
    R^+_\rho
    =
    \Pr
    \left[
    a^{(\rho)} = a^\star
    \mid
    a^{(0)} = a^\star
    \right].
\end{equation*}
Same wrong preservation measures, whether regeneration repeats the same
wrong answer when the initial answer is wrong, is defined as:
\begin{equation*}
    R^-_\rho
    =
    \Pr
    \left[
    a^{(\rho)} = a^{(0)}
    \mid
    a^{(0)}\neq\varnothing,\;
    a^{(0)} \not= a^\star
    \right].
\end{equation*}

\begin{mainobservation}[Reproduction rates]\label{obs:cons}
Across the different experiment settings, correct initial answers are
more likely to be preserved than incorrect initial answers are to be
repeated in DLMs:
\begin{equation*}
      R^+_\rho \geq R^-_\rho .
\end{equation*}
\end{mainobservation}

Using Observation~\ref{obs:cons}, \method{} assesses the reliability of the initial answer. \method{} votes only over regenerated answers, initial answer
$a^{(0)}$ is excluded from the final vote. This isolates the effect of prefix-conditioned regeneration: the DLMs can improve over the initial answer only if the regenerated candidates support a better answer. The multiset of non-null regenerated answers is:
\begin{equation*}
    \mathcal B
    =
    \left\{
        a^{(\rho)}
        :
        \rho\in\mathcal R,\;
        a^{(\rho)}\neq\varnothing
    \right\}
\end{equation*}
\noindent
If $\mathcal B=\{\}$, the example is marked as no-vote. Otherwise,
\method{} returns the most frequent regenerated answer, denoted by
$\Maj(\mathcal B)$:
\begin{equation*}
    \hat a
    =
    \Maj(\mathcal B).
\end{equation*}
The three keep rates probe different neighborhoods around the same initial output. The keep rate of $0.1$ imposes only weak conditioning on the original generation trajectory and allows most of the reasoning path to change. The keep rate of $0.9$ largely preserves the answer while perturbing only the final portion of the output. The keep rate of $0.5$ provides an intermediate between the two. Together, these rates evaluate answer stability under varying strengths of prefix conditioning.

\paragraph{When does \method{} improve accuracy?}
To demonstrate the effectiveness of PDC, we consider a stylized setting with only two possible answers: the correct answer
\(a^\star\) and one incorrect answer \(b\). The purpose of this analysis
is to identify when taking a majority vote over three regenerated
answers improves upon the initial answer.

First, consider three independent Bernoulli variables (i.e., variables that take either one or zero) whose
probabilities of being one are \(p_1,p_2,p_3\), respectively. The probability that at least two of them are one is
\begin{align}
    V_3(\boldsymbol p)
    &:=
    \Pr[\text{at least two of the three variables are one}]
    \nonumber\\
    &=
    p_1p_2(1-p_3)
    +p_1(1-p_2)p_3
    +(1-p_1)p_2p_3
    +p_1p_2p_3
    \nonumber\\
    &=
    p_1p_2+p_1p_3+p_2p_3-2p_1p_2p_3,
    \label{eq:three-vote-probability}
\end{align}
where \(\boldsymbol p=(p_1,p_2,p_3)\).

When all three probabilities are equal to \(p\), we write
\[
    V_3(p):=V_3(p,p,p)=3p^2-2p^3.
\]
Thus, \(V_3\) is the probability that the majority outcome is one.

For regeneration \(j\), define
\[
    R^+_{\rho_j}
    :=
    \Pr[a^{(j)}=a^{(0)}
        \mid a^{(0)}=a^\star]
\]
and
\[
    R^-_{\rho_j}
    :=
    \Pr[a^{(j)}=a^{(0)}
        \mid  a^{(0)}\neq\varnothing,\;a^{(0)}=b].
\]
In words, \(R^+_{\rho_j}\) is the correct-answer retention probability, while \(R^-_{\rho_j}\) is the same-wrong-answer persistence probability.

\begin{maintheorem}[Benefit of the reproduction gap]
\label{thm:reproduction-gap}
Assume a binary answer space
\(\mathcal A=\{a^\star,b\}\) with no null answers, and let
\[
    \pi:=\Pr[a^{(0)}=a^\star].
\]
Suppose that, conditional on \(a^{(0)}\), the three
regenerated answers are independent. Define
\[
    \boldsymbol R^\pm
    :=
    \left(
        R^\pm_{\rho_1},
        R^\pm_{\rho_2},
        R^\pm_{\rho_3}
    \right).
\]
Then
\begin{equation}
    \Acc(\method{})
    =
    \pi V_3(\boldsymbol R^+)
    +(1-\pi)\left[1-V_3(\boldsymbol R^-)\right].
    \label{eq:pdc-binary-accuracy}
\end{equation}
In the non-degenerate case, \(\method{}\) improves upon \(\init{}\) if
and only if
\begin{equation}
    \pi
    <
    \frac{1-V_3(\boldsymbol R^-)}
    {2-V_3(\boldsymbol R^-)-V_3(\boldsymbol R^+)}.
    \label{eq:pdc-improvement-condition}
\end{equation}

In particular, if \(\pi=1/2\), then
\begin{equation}
    \Acc(\method{})-\Acc(\init{})
    =
    \frac12\left[
        V_3(\boldsymbol R^+)-V_3(\boldsymbol R^-)
    \right].
    \label{eq:pdc-gap-direct}
\end{equation}
Consequently,
\[
    V_3(\boldsymbol R^+)>V_3(\boldsymbol R^-)
    \quad\Longleftrightarrow\quad
    \Acc(\method{})>\Acc(\init{}).
\]
Moreover, because \(V_3\) is increasing on
each coordinate, a sufficient condition for improvement is
\[
    R^+_{\rho_j}\geq R^-_{\rho_j},
    \qquad j=1,2,3,
\]
with strict inequality for at least one \(j\).
\end{maintheorem}

\begin{proof}
We separate the analysis into two cases.

\paragraph{Case 1: the initial answer is correct.}
Suppose \(a^{(0)}=a^\star\). Regeneration \(j\) reproduces the initial
answer, which is also the correct answer, with probability
\(R^+_{\rho_j}\). By conditional independence, the probability that at
least two of the three regenerated answers are correct is therefore
\[
    \Pr[\hat a=a^\star\mid a^{(0)}=a^\star]
    =
    V_3(\boldsymbol R^+).
\]

\paragraph{Case 2: the initial answer is wrong.}
Suppose \(a^{(0)}=b\). Because there are only two possible answers, a
regeneration is correct exactly when it does not reproduce the initial
answer.

The final majority vote is wrong exactly when at least two
regenerations reproduce \(b\). The probability of this event is
\(V_3(\boldsymbol R^-)\). Hence,
\[
    \Pr[\hat a=a^\star\mid a^{(0)}=b]
    =
    1-V_3(\boldsymbol R^-).
\]

The initial answer is correct with probability \(\pi\) and wrong with
probability \(1-\pi\). Combining the two cases gives
\[
    \Acc(\method{})
    =
    \pi V_3(\boldsymbol R^+)
    +(1-\pi)\left[1-V_3(\boldsymbol R^-)\right],
\]
which proves Eq.~\eqref{eq:pdc-binary-accuracy}.

Since the accuracy of the initial answer is
\(\Acc(\init{})=\pi\), the change in accuracy is
\begin{align}
    \Acc(\method{})-\Acc(\init{})
    &=
    (1-\pi)\left[1-V_3(\boldsymbol R^-)\right]
    \nonumber\\
    &\quad
    -\pi\left[1-V_3(\boldsymbol R^+)\right].
    \label{eq:pdc-gain-intuitive}
\end{align}

This expression has a direct interpretation:
\[
    \underbrace{
        (1-\pi)\left[1-V_3(\boldsymbol R^-)\right]
    }_{\text{wrong initial answers that are corrected}}
    -
    \underbrace{
        \pi\left[1-V_3(\boldsymbol R^+)\right]
    }_{\text{correct initial answers that are spoiled}}.
\]
Thus, the method improves accuracy precisely when the benefit from
correcting initially wrong answers exceeds the loss from spoiling
initially correct answers.
Rearranging
\(\Acc(\method{})-\Acc(\init{})>0\) gives
Eq.~\eqref{eq:pdc-improvement-condition}.
Finally, suppose that \(\pi=1/2\). Then
\[
    \Acc(\method{})-\Acc(\init{})
    =
    \frac12\left[
        V_3(\boldsymbol R^+)-V_3(\boldsymbol R^-)
    \right],
\]
which gives Eq.~\eqref{eq:pdc-gap-direct} and the stated equivalence.
It remains to establish the sufficient condition. From
Eq.~\eqref{eq:three-vote-probability},
\begin{align*}
    \frac{\partial V_3}{\partial p_1}
    &=p_2+p_3-2p_2p_3 \\
    &=p_2(1-p_3)+p_3(1-p_2)>0
\end{align*}
for \(p_2,p_3\in(0,1)\). The same argument applies to the other two
coordinates, so \(V_3\) is strictly increasing in each coordinate on
\((0,1)^3\). Therefore, if
\(R^+_{\rho_j}\geq R^-_{\rho_j}\) for every \(j\), with strict
inequality for at least one \(j\), then
\[
    V_3(\boldsymbol R^+)>V_3(\boldsymbol R^-),
\]
\end{proof}



\section{Experiments}

The experiments first examine the diagnostic signal behind \method{}:
whether prefix regeneration preserves initially correct answers more
often than it repeats the same wrong answer. We then evaluate \method{} against the initial sample and temporal voting (\tif{} \citep{wang2026time})
baselines in full diffusion setting \footnote{In the LLaDA family, the block length controls the generation. When the block length is $1$, the model reduces to an AR model; when the block length is $L$, it corresponds to the full diffusion setting; and when the block length is between $1$ and $L$, it yields a semi-AR setting.
}. In the ablation study, we compare \method{} with standard majority voting that spends more token budget than it.
Furthermore, we had several robustness checks, including several alternative unmasking strategies.

\paragraph{Models and datasets.}
We study Dream-7B (\texttt{Dream-org/Dream-v0-Instruct-7B},
\citet{ye2025dream}) and LLaDA-family models
(\texttt{GSAI-ML/LLaDA-8B-Instruct}, \citet{nie2025llada};
\texttt{GSAI-ML/LLaDA-1.5}, \citet{zhu2025llada15}). 
The main results use the full-diffusion setting. For LLaDA-family models, we set the block length equal to the generation length. All models are decoded at temperature $0.2$.
The datasets we tested are
GSM8K \citep{cobbe2021gsm8k}, MATH-500
\citep{hendrycks2021math,lightman2024verify}, SVAMP
\citep{patel2021svamp}, CSQA (CommonsenseQA,
\citet{talmor-etal-2019-commonsenseqa}), and SQA (StrategyQA,
\citet{geva2021strategyqa}). 

\begin{table}[t]
\centering
\scriptsize
\caption{
For each pair, we select the generation length that
yields the highest \init{} accuracy; this selection is based only on \init{} and is not optimized for
\method{}. \method{} attains the best or tied-best accuracy in all 15 settings and improves over the selected
\init{} baseline in 14 of them, with gains of up to $+5.68$ accuracy
points. The full result is reported in
Appendix~\ref{app:full-diffusion-results}. Signed green/red values denote the advantage of the method relative to the selected \init{} baseline. Bold
indicates the best accuracy in each row, and an underline indicates the
second best accuracy, with ties marked together.
}

\label{tab:main-results}
\begin{adjustbox}{max width=\textwidth}
\begin{tabular}{llcrrrrrrr}
\toprule
\textbf{Model} & \textbf{Dataset} & \textbf{Length}
& \init{} & \tif{} Fixed & \tif{} Linear & \tif{} Exp.
& \method{} & $\Delta$ \tif{} avg. & $\Delta$ \method{} \\
\midrule

\multirow{5}{*}{Dream-7B}
& GSM8K & 512 & \underline{82.03} & 81.96 & \underline{82.03} & \underline{82.03} & \textbf{84.08} & \neggain{-0.02} & \posgain{+2.05} \\
& MATH-500 & 512 & 46.00 & 46.00 & 46.00 & \underline{46.20} & \textbf{47.00} & \posgain{+0.07} & \posgain{+1.00} \\
& SVAMP & 256 & \underline{87.00} & \underline{87.00} & \underline{87.00} & \underline{87.00} & \textbf{88.00} & \posgain{+0.00} & \posgain{+1.00} \\
& CSQA & 128 & 73.22 & 73.79 & \underline{73.96} & 73.55 & \textbf{74.20} & \posgain{+0.55} & \posgain{+0.98} \\
& SQA & 128 & 70.45 & \underline{70.74} & \underline{70.74} & \underline{70.74} & \textbf{71.76} & \posgain{+0.29} & \posgain{+1.31} \\
\midrule

\multirow{5}{*}{LLaDA-8B}
& GSM8K & 256 & 60.73 & 60.12 & \underline{60.80} & \underline{60.80} & \textbf{65.88} & \neggain{-0.15} & \posgain{+5.15} \\
& MATH-500 & 256 & \underline{26.45} & 26.25 & 26.25 & 26.25 & \textbf{27.86} & \neggain{-0.20} & \posgain{+1.41} \\
& SVAMP & 128 & \underline{83.67} & \underline{83.67} & \underline{83.67} & \underline{83.67} & \textbf{84.33} & \posgain{+0.00} & \posgain{+0.66} \\
& CSQA & 128 & \underline{80.51} & \underline{80.51} & \underline{80.51} & \underline{80.51} & \textbf{80.59} & \posgain{+0.00} & \posgain{+0.08} \\
& SQA & 128 & \underline{65.07} & \underline{65.07} & \underline{65.07} & \underline{65.07} & \textbf{66.67} & \posgain{+0.00} & \posgain{+1.60} \\
\midrule

\multirow{5}{*}{LLaDA-1.5}
& GSM8K & 256 & \underline{60.58} & 60.05 & 60.27 & 60.50 & \textbf{66.26} & \neggain{-0.30} & \posgain{+5.68} \\
& MATH-500 & 256 & \underline{25.65} & \underline{25.65} & 25.45 & \underline{25.65} & \textbf{26.65} & \neggain{-0.07} & \posgain{+1.00} \\
& SVAMP & 128 & \underline{83.67} & 83.33 & \underline{83.67} & \underline{83.67} & \textbf{84.67} & \neggain{-0.11} & \posgain{+1.00} \\
& CSQA & 128 & \textbf{80.10} & \textbf{80.10} & \textbf{80.10} & \textbf{80.10} & \textbf{80.10} & \posgain{+0.00} & \posgain{+0.00} \\
& SQA & 128 & \underline{66.08} & \underline{66.08} & \underline{66.08} & \underline{66.08} & \textbf{66.81} & \posgain{+0.00} & \posgain{+0.73} \\
\bottomrule
\end{tabular}
\end{adjustbox}
\end{table}

\subsection{Correct Answers Are More Reproducible Than Wrong Answers}\label{subsec:stability}

We first examine the reproduction rates that motivate \method{}.
Figure~\ref{fig:regen-consistency-rho50} shows reproduction rates at
keep rate $0.5$; the results for keep rates $0.1$ or $0.9$ across all datasets are reported
in Appendix Table~\ref{tab:regen-consistency-all}. The figure shows a
significant gap between correct-answer retention and the same wrong answer persistence: initially correct answers are typically reproduced under
prefix-conditioned regeneration, whereas initially, wrong answers are
less likely to be regenerated as the same wrong answer.

\newcommand{\fullstabilitytable}{%
\begin{table*}[p]
  \centering
  \caption{
  Reproducing rate across the full model and benchmarks.
  We report correct-answer retention $R^+_\rho$, same-wrong persistence
  $R^-_\rho$, and their separation $\Delta = R^+_\rho - R^-_\rho$. Bold marks the largest $\Delta$ within each
  row; negative $\Delta$ values are shown in red.
}
  \label{tab:regen-consistency-all}

  \tiny
  \setlength{\tabcolsep}{2.3pt}
  \renewcommand{\arraystretch}{1.03}

  \sisetup{
    table-number-alignment=center,
    detect-weight=true,
    detect-inline-weight=math
  }

  \begin{threeparttable}
  \begin{adjustbox}{max width=\textwidth}
  \begin{tabular}{@{}llc
                  S[table-format=3.1]
                  S[table-format=3.1]
                  S[table-format=3.1]
                  @{\hspace{6pt}}
                  S[table-format=3.1]
                  S[table-format=3.1]
                  S[table-format=3.1]
                  @{\hspace{6pt}}
                  rrr
                  @{}}
    \toprule
    \multirow{2}{*}{\textbf{Dataset}}
    & \multirow{2}{*}{\textbf{Model}}
    & \multirow{2}{*}{\textbf{Length}}
    & \multicolumn{3}{c}{$R^+_\rho$}
    & \multicolumn{3}{c}{$R^-_\rho$}
    & \multicolumn{3}{c}{$\Delta$} \\
    \cmidrule(lr){4-6}
    \cmidrule(lr){7-9}
    \cmidrule(l){10-12}
    & &
    & {0.1} & {0.5} & {0.9}
    & {0.1} & {0.5} & {0.9}
    & {0.1} & {0.5} & {0.9} \\
    \midrule

    \multirow{9}{*}{GSM8K}
    & \multirow{3}{*}{Dream-7B}
    & 128 & 97.9 & 98.3 & 99.2 & 82.5 & 82.5 & 93.3 & 15.3 & \textbf{15.8} & 5.8 \\
    & & 256 & 97.8 & 99.3 & 99.8 & 83.5 & 84.4 & 91.3 & 14.3 & \textbf{15.0} & 8.5 \\
    & & 512 & 97.5 & 99.3 & 99.5 & 86.8 & 82.8 & 90.7 & 10.7 & \textbf{16.4} & 8.9 \\
    \cmidrule(lr){2-12}

    & \multirow{3}{*}{LLaDA-1.5}
    & 128 & 92.6 & 94.4 & 99.2 & 42.8 & 36.0 & 86.4 & 49.8 & \textbf{58.4} & 12.8 \\
    & & 256 & 95.2 & 95.0 & 99.6 & 65.9 & 22.0 & 84.7 & 29.4 & \textbf{73.0} & 14.9 \\
    & & 512 & 88.8 & 93.0 & 98.4 & 44.8 & 26.1 & 75.9 & 44.0 & \textbf{66.9} & 22.5 \\
    \cmidrule(lr){2-12}

    & \multirow{3}{*}{LLaDA-8B}
    & 128 & 93.2 & 94.4 & 99.1 & 43.9 & 37.1 & 90.1 & 49.4 & \textbf{57.3} & 9.0 \\
    & & 256 & 93.0 & 95.0 & 99.6 & 69.5 & 28.3 & 87.9 & 23.5 & \textbf{66.7} & 11.8 \\
    & & 512 & 80.3 & 87.6 & 99.2 & 48.5 & 26.3 & 77.8 & 31.8 & \textbf{61.3} & 21.4 \\

    \midrule

    \multirow{9}{*}{MATH-500}
    & \multirow{3}{*}{Dream-7B}
    & 128 & 90.5 & 92.2 & 97.8 & 78.3 & 74.1 & 91.5 & 12.2 & \textbf{18.1} & 6.3 \\
    & & 256 & 93.5 & 93.5 & 98.6 & 72.0 & 73.0 & 83.1 & \textbf{21.5} & 20.6 & 15.5 \\
    & & 512 & 95.7 & 95.2 & 97.8 & 75.2 & 78.6 & 79.5 & \textbf{20.4} & 16.7 & 18.3 \\
    \cmidrule(lr){2-12}

    & \multirow{3}{*}{LLaDA-1.5}
    & 128 & 83.8 & 83.8 & 96.6 & 55.6 & 48.0 & 91.1 & 28.2 & \textbf{35.7} & 5.5 \\
    & & 256 & 88.3 & 84.4 & 99.2 & 69.3 & 46.6 & 94.1 & 19.0 & \textbf{37.8} & 5.2 \\
    & & 512 & 86.2 & 85.3 & 94.5 & 61.8 & 31.0 & 81.6 & 24.5 & \textbf{54.3} & 12.9 \\
    \cmidrule(lr){2-12}

    & \multirow{3}{*}{LLaDA-8B}
    & 128 & 79.8 & 80.7 & 97.5 & 54.2 & 49.2 & 88.7 & 25.6 & \textbf{31.5} & 8.8 \\
    & & 256 & 88.6 & 82.6 & 97.7 & 73.0 & 44.8 & 92.6 & 15.7 & \textbf{37.8} & 5.1 \\
    & & 512 & 81.6 & 74.7 & 95.4 & 63.5 & 34.7 & 81.7 & 18.1 & \textbf{40.0} & 13.7 \\

    \midrule

    \multirow{9}{*}{SVAMP}
    & \multirow{3}{*}{Dream-7B}
    & 128 & 98.8 & 98.8 & 99.2 & 74.4 & 71.8 & 79.5 & 24.4 & \textbf{27.0} & 19.7 \\
    & & 256 & 98.9 & 99.2 & 100.0 & 81.6 & 68.4 & 86.8 & 17.3 & \textbf{30.8} & 13.2 \\
    & & 512 & 98.8 & 99.6 & 100.0 & 82.9 & 80.0 & 82.9 & 16.0 & \textbf{19.6} & 17.1 \\
    \cmidrule(lr){2-12}

    & \multirow{3}{*}{LLaDA-1.5}
    & 128 & 97.2 & 98.8 & 99.6 & 73.5 & 49.0 & 95.9 & 23.7 & \textbf{49.8} & 3.7 \\
    & & 256 & 93.4 & 94.7 & 99.1 & 79.2 & 33.3 & 83.3 & 14.2 & \textbf{61.4} & 15.8 \\
    & & 512 & 94.1 & 98.0 & 100.0 & 78.6 & 64.3 & 92.9 & 15.6 & \textbf{33.8} & 7.1 \\
    \cmidrule(lr){2-12}

    & \multirow{3}{*}{LLaDA-8B}
    & 128 & 98.0 & 96.8 & 100.0 & 75.0 & 62.5 & 95.8 & 23.0 & \textbf{34.3} & 4.2 \\
    & & 256 & 89.4 & 95.6 & 99.4 & 75.0 & 38.6 & 95.5 & 14.4 & \textbf{57.0} & 3.9 \\
    & & 512 & 96.2 & 94.3 & 100.0 & 100.0 & 66.7 & 66.7 & -3.8 & 27.7 & \textbf{33.3} \\

    \midrule

    \multirow{9}{*}{CSQA}
    & \multirow{3}{*}{Dream-7B}
    & 128 & 98.9 & 96.2 & 97.5 & 96.8 & 87.5 & 91.4 & 2.1 & \textbf{8.7} & 6.2 \\
    & & 256 & 98.1 & 95.9 & 97.0 & 95.9 & 87.6 & 93.0 & 2.2 & \textbf{8.3} & 4.0 \\
    & & 512 & 98.8 & 95.4 & 98.5 & 95.1 & 87.5 & 93.1 & 3.7 & \textbf{7.9} & 5.5 \\
    \cmidrule(lr){2-12}

    & \multirow{3}{*}{LLaDA-1.5}
    & 128 & 97.9 & 96.9 & 99.3 & 86.0 & 89.7 & 97.1 & \textbf{11.8} & 7.2 & 2.2 \\
    & & 256 & 96.6 & 98.3 & 99.8 & 88.1 & 89.0 & 99.6 & 8.5 & \textbf{9.3} & 0.2 \\
    & & 512 & 97.5 & 97.7 & 99.9 & 88.5 & 88.9 & 100.0 & \textbf{9.1} & 8.8 & -0.1\\
    \cmidrule(lr){2-12}

    & \multirow{3}{*}{LLaDA-8B}
    & 128 & 97.6 & 98.6 & 99.7 & 88.7 & 90.8 & 97.5 & \textbf{8.9} & 7.8 & 2.2 \\
    & & 256 & 95.2 & 97.7 & 99.8 & 85.5 & 88.0 & 100.0 & \textbf{9.7} & 9.7 & -0.2 \\
    & & 512 & 98.1 & 97.5 & 100.0 & 81.6 & 85.1 & 100.0 & \textbf{16.6} & 12.4 & 0.0 \\
    
    \midrule

    \multirow{9}{*}{SQA}
    & \multirow{3}{*}{Dream-7B}
    & 128 & 98.6 & 95.7 & 96.9 & 92.4 & 91.9 & 97.8 & \textbf{6.1} & 3.8 & -0.9\\
    & & 256 & 97.4 & 93.8 & 96.0 & 91.7 & 85.4 & 94.3 & 5.8 & \textbf{8.4} & 1.7 \\
    & & 512 & 98.1 & 92.5 & 96.6 & 92.7 & 87.5 & 97.4 & \textbf{5.4} & 5.0 & -0.8 \\
    \cmidrule(lr){2-12}

    & \multirow{3}{*}{LLaDA-1.5}
    & 128 & 93.6 & 97.1 & 99.3 & 88.0 & 92.7 & 98.7 & \textbf{5.6} & 4.4 & 0.6 \\
    & & 256 & 94.3 & 96.4 & 98.5 & 87.1 & 89.1 & 95.0 & 7.3 & \textbf{7.4} & 3.5 \\
    & & 512 & 96.8 & 99.1 & 100.0 & 86.2 & 91.9 & 97.6 & \textbf{10.6} & 7.2 & 2.4 \\
    \cmidrule(lr){2-12}

    & \multirow{3}{*}{LLaDA-8B}
    & 128 & 94.6 & 96.4 & 99.6 & 87.1 & 86.3 & 97.5 & 7.5 & \textbf{10.2} & 2.1 \\
    & & 256 & 87.5 & 96.7 & 99.2 & 90.8 & 84.2 & 93.4 & -3.3 & \textbf{12.5} & 5.8 \\
    & & 512 & 93.8 & 100.0 & 100.0 & 84.6 & 92.3 & 100.0 & \textbf{9.1} & 7.7 & 0.0 \\
    \bottomrule
  \end{tabular}
  \end{adjustbox}
  \end{threeparttable}
\end{table*}
}

\subsection{\method{} Improves the Initial Sample}\label{subsec:acc}
We next examine whether PDC actually improves the accuracy of the answer. This is non-trivial: regenerations can move away from an initial wrong
answer, but the new answers may still be wrong, and the final vote may
still result in an incorrect answer.
\paragraph{Baselines.}
\init{} indicates the initial answer $y$ generated by the standard denoising process. \method{} votes over
0.1, 0.5, and 0.9 keep-rate regenerations. Ties are broken by the fixed keep-rate order $0.1,0.5,0.9$. \tif{}, the state-of-the-art method for improving DLLM accuracy, extracts answers
from intermediate steps.
To aggregate answers on these steps, \tif{} Linear and \tif{} Exp use linearly and exponentially increasing
temporal weights, respectively. The exponential setting uses $\alpha=5$,
following the settings of
temporal voting in DLMs \citep{wang2026time}.

\paragraph{Results}
Table~\ref{tab:main-results} compares the accuracy of the methods. For
each model and dataset pair, we optimize the generation length $L$ to achieve the highest \init{} accuracy. This
compares \method{} against the strongest available initial sample
baseline for that pair, rather than against a favorable length chosen
post hoc for \method{}. Results across several values of $L$ are reported in
Appendix~\ref{app:full-diffusion-results}.

On Dream-7B, \method{} consistently improves over \init{} on all five datasets.
The advantage of \method{} over \init{} ranges from +0.98 to +2.05 points, with an unweighted mean gain of +1.27 points. On LLaDA-family models, \method{} improves over
\init{} in nine of ten model--dataset pairs and ties in the remaining
pair, with an unweighted mean gain of +1.73 points. 
Appendix~\ref{app:semi-ar-llada} reports additional LLaDA-family runs
with block length $32$, a semi-AR setting. It shows that the same
prefix-regeneration signal remains useful.

\subsection{Ablation Study}

\subsubsection{\method{} Outperforms Self-Consistency at Lower Inference Cost}

We compare \method{} with Self-Consistency at generation length $L=128$. 
The \init{}~x4 baseline conducts a majority vote over four independent generations and has a total denoising-step budget of $4T$. This budget is approximately $2.5T$ when the initial output fills the length-$L$ window, can exceed $3T$ for shorter outputs, and remains below $4T$ for every nonempty initial output. Thus, \init{}~x4 is a conservative higher-budget baseline rather than an exactly cost-matched baseline.
Improvements over \init{} x4 suggest that the method is using regeneration more effectively rather than relying merely on self-consistency.
Table~\ref{tab:compute-fair-l128} shows the completed
$L=128$ runs.
In this setting, \method{} improves over \init{}~x4 in 13 of
the 15 completed model and dataset pairs, ties in one pair, and underperforms
in one pair, indicating that the gains stem from a more effective regeneration strategy that utilizes the consistency of the reasoning process, rather than merely drawing additional full generations.

\begin{table*}[h]
\centering
\caption{
We compare \method{} with Self-Consistency under different total denoising-step budgets at generation length $L=128$. For both \init{} and \tif{}, $\times4$ denotes four independent initial denoising trajectories. \init{}~x4 votes over their final answers with budget $4T$, while \method{} uses one initial generation followed by three regenerations at keep rates $0.1$, $0.5$, and $0.9$. \method{} achieves the best or tied-best
accuracy in 14 of 15 model--dataset settings and improves over
\init{}~x4 in 13 of them, with gains of up to $+9.25$ accuracy points.
}
\label{tab:compute-fair-l128}
\scriptsize
\setlength{\tabcolsep}{3.5pt}
\begin{adjustbox}{max width=\textwidth}
\begin{tabular}{llrrrrrrr}
\toprule
\textbf{Dataset}
& \textbf{Model}
& \init{}
& \init{}~x4
& \tif{} Fixed x4
& \tif{} Linear x4
& \tif{} Exp. x4
& \method{}
& $\Delta$ vs. \init{}~x4 \\
\midrule
\multirow{3}{*}{GSM8K}
& Dream-7B & 64.06 & 64.22 & 62.40 & 64.52 & 65.81 & \textbf{66.03} & \posgain{+1.81} \\
& LLaDA-1.5 & 56.48 & 56.63 & 56.41 & 56.56 & 56.71 & \textbf{65.88} & \posgain{+9.25} \\
& LLaDA-8B & 58.30 & 58.91 & 58.38 & 58.61 & 58.91 & \textbf{65.88} & \posgain{+6.97} \\
\midrule
\multirow{3}{*}{MATH-500}
& Dream-7B & 35.80 & 36.00 & 34.20 & 34.80 & 35.60 & \textbf{37.20} & \posgain{+1.20} \\
& LLaDA-1.5 & 23.40 & 23.20 & 23.40 & 23.40 & 23.00 & \textbf{26.40} & \posgain{+3.20} \\
& LLaDA-8B & 23.80 & 23.60 & 23.80 & 23.80 & 23.60 & \textbf{25.40} & \posgain{+1.80} \\
\midrule
\multirow{3}{*}{SVAMP}
& Dream-7B & 83.00 & 83.00 & 84.00 & 83.67 & 84.00 & \textbf{84.33} & \posgain{+1.33} \\
& LLaDA-1.5 & 83.67 & 83.67 & 83.33 & 83.67 & 83.67 & \textbf{84.67} & \posgain{+1.00} \\
& LLaDA-8B & 83.67 & \textbf{84.33} & \textbf{84.33} & 84.00 & \textbf{84.33} & \textbf{84.33} & \posgain{+0.00} \\
\midrule
\multirow{3}{*}{CSQA}
& Dream-7B & 73.22 & 73.22 & 73.79 & 73.96 & 73.55 & \textbf{74.20} & \posgain{+0.98} \\
& LLaDA-1.5 & 80.10 & 80.51 & \textbf{80.67} & \textbf{80.67} & \textbf{80.67} & 80.10 & \neggain{-0.41} \\
& LLaDA-8B & 80.51 & 80.34 & \textbf{80.59} & \textbf{80.59} & \textbf{80.59} & \textbf{80.59} & \posgain{+0.25} \\
\midrule
\multirow{3}{*}{SQA}
& Dream-7B & 70.45 & 70.45 & 70.74 & 70.74 & 70.74 & \textbf{71.76} & \posgain{+1.31} \\
& LLaDA-1.5 & 66.08 & 65.94 & 65.94 & 65.79 & 65.79 & \textbf{66.81} & \posgain{+0.87} \\
& LLaDA-8B & 65.07 & 64.63 & 65.07 & 65.07 & 65.07 & \textbf{66.67} & \posgain{+2.04} \\
\bottomrule
\end{tabular}
\end{adjustbox}
\end{table*}

\subsubsection{Robustness}

We now conduct robustness checks. First, we tested several variants of the unmasking strategy. Up to this point, the results have used entropy
unmasking for Dream and low-confidence unmasking for LLaDA-family models; Table~\ref{tab:robustness-remasking} repeats the MATH-500 evaluation with alternative unmasking rules, including random and origin unmasking.
Table~\ref{tab:robustness-remasking} shows that larger keep rates also retain more initially
correct answers, while smaller keep rates more often break exact
repetition of the initial wrong answer. Accuracy remains comparable to
or above the corresponding initial generation in these settings.

In addition, we conduct two other robustness checks.
Appendix~\ref{app:semi-ar-llada} reports LLaDA-family experiments with a block length of 32, corresponding to semi-AR inference. The results show that \method{} remains effective under this semi-AR setting.
Moreover, we test whether the advantage of \method{} is specific to the low-temperature decoding regime by repeating the $L=128$ comparison under more stochastic decoding at $\tau=1.0$. At this temperature, \method{} attains the highest accuracy in 12 of the 15 model and dataset settings and exceeds \init{}~$\times 4$ by $4.98$ accuracy points on average (see Appendix~\ref{app:high-temp}).
These runs are not used for the main full-diffusion claims, but they check
whether the prefix-conditioning signal also appears under a different
inference time schedule.
Together, these experiments do not make \method{}
independent of the DLM, but they indicate that the signal is not
specific to one unmasking heuristic or one block schedule.

\begin{table*}[h]
  \centering
  \caption{
Robustness to alternative unmasking rules on MATH-500 across models and generation lengths. Dream uses origin unmasking, while the LLaDA-family models use random remasking. All rows use the full-diffusion setting with block length equal to generation length. We show answer accuracy for the initial generation, \method{} with keep rates $\rho\in\{0.1,0.5,0.9\}$, and \tif{} variants. \method{} achieves the best
accuracy in all models and length settings.
}
  \label{tab:robustness-remasking}

  \scriptsize
  \setlength{\tabcolsep}{2.2pt}
  \renewcommand{\arraystretch}{1.08}

  \sisetup{
    table-number-alignment=center,
    detect-weight=true,
    detect-inline-weight=math
  }

  \begin{threeparttable}
  \begin{adjustbox}{max width=\textwidth}
  \begin{tabular}{@{}ll
                  S[table-format=2.2]
                  S[table-format=2.2]
                  S[table-format=2.2]
                  S[table-format=2.2]
                  S[table-format=2.2]
                  S[table-format=2.2]
                  S[table-format=2.2]
                  S[table-format=2.2]
                  S[table-format=2.2]
                  S[table-format=2.2]
                  S[table-format=2.2]@{}}
    \toprule
    \textbf{Model}
    & \textbf{Length}
    & \multicolumn{5}{c}{\textbf{Accuracy (\%)}} 
    & \multicolumn{6}{c}{\textbf{Reproduction rates (\%)}} \\
    \cmidrule(lr){3-7}
    \cmidrule(l){8-13}

    &
    & \multicolumn{2}{c}{\textbf{}}
    & \multicolumn{3}{c}{\textbf{\tif{} variants}}
    & \multicolumn{2}{c}{$\rho=0.1$}
    & \multicolumn{2}{c}{$\rho=0.5$}
    & \multicolumn{2}{c}{$\rho=0.9$} \\
    \cmidrule(lr){3-4}
    \cmidrule(lr){5-7}
    \cmidrule(lr){8-9}
    \cmidrule(lr){10-11}
    \cmidrule(l){12-13}

    &
    & {\init{}}
    & {\method{}}
    & {Fixed}
    & {Linear}
    & {Exp.}
    & {$R^+_\rho$}
    & {$R^-_\rho$}
    & {$R^+_\rho$}
    & {$R^-_\rho$}
    & {$R^+_\rho$}
    & {$R^-_\rho$} \\
    \midrule

    \multirow{3}{*}{Dream-7B}
    & 128 & 17.60 & \textbf{23.40} & 17.60 & \underline{17.80} & 17.60 & 51.14 & 12.66 & 64.77 & 17.53 & 95.45 & 64.61 \\
    & 256 & 14.20 &\textbf{20.40} & 14.80 & 14.80 & \underline{15.00} & 60.56 & 10.74 & 56.34 & 17.18 & 78.87 & 57.67 \\
    & 512 & 14.80 & \textbf{20.40} & \underline{15.20} & \underline{15.20} & 15.00 & 54.05 & 14.78 & 54.05 & 19.81 & 83.78 & 58.18 \\
    \addlinespace[2pt]

    \multirow{3}{*}{LLaDA-1.5}
    & 128 & 27.20 & \textbf{28.60} & 27.80 & \underline{28.20} & 27.20 & 68.38 & 21.25 & 73.53 & 36.26 & 94.85 & 81.30 \\
    & 256 & 31.46 & \textbf{33.07} & 31.66 & 31.86 & \underline{32.06} & 70.06 & 18.34 & 77.07 & 30.77 & 92.99 & 80.47 \\
    & 512 & 31.45 & \textbf{36.29} & 31.25 & 32.66 & \underline{32.86} & 71.79 & 11.68 & 73.08 & 23.65 & 90.38 & 73.35 \\
    \addlinespace[2pt]

    \multirow{3}{*}{LLaDA-8B}
    & 128 & 25.20 & \textbf{28.00} & 25.00 & \underline{25.40} & \underline{25.40} & 64.29 & 18.31 & 77.78 & 34.37 & 93.65 & 76.34 \\
    & 256 & 31.06 & \textbf{33.27} & 29.46 & 30.86 & \underline{31.66} & 69.03 & 22.71 & 76.77 & 33.92 & 94.19 & 79.65 \\
    & 512 & 30.85 & \textbf{34.68} & \underline{31.65} & 31.05 & 31.45 & 65.36 & 14.08 & 80.39 & 23.17 & 95.42 & 65.98 \\
    \bottomrule
  \end{tabular}
  \end{adjustbox}
  \end{threeparttable}
\end{table*}

\section{Conclusion}

In this work, we presented \method{} (\emph{Prefix-Denoising Consistency}), a test-time self-verification method. \method{} exploits a distinctive consistency signal from prefix-conditioned regeneration, correct reasoning trajectories tend to be more stable and reproducible than incorrect ones. 
Across math and commonsense benchmarks, \method{} consistently improved over the initial sample, compared favorably to independent generations under a compute-constrained setting, and remained robust across unmasking strategies and hyperparameter choices. These results suggest that prefix-conditioned regeneration is an effective DLM-specific primitive for test-time verification, and point to the broader potential of exploiting denoising-based regeneration signals to improve the reliability of diffusion language models.

\ificlrfinal 
\section*{Acknowledgments}
J. Komiyama was supported by the MBZUAI Start-up Fund [BF0121].
\fi
\bibliographystyle{iclr2026_conference}
\bibliography{references}


\appendix

\section{Related Work}
\label{app:related-work}

\paragraph{Diffusion language models.}
Diffusion models for text generation succeeded by \citet{austin2021}, advanced through the masked token framework. Recent DLM work has improved both the probabilistic formulation and the
scale of masked/discrete diffusion. Score-entropy discrete diffusion,
masked diffusion language modeling, and simplified masked diffusion
objectives improve training and likelihood modeling for token sequences
\citep{lou2024sedd,sahoo2024mdlm,shi2024md4}. At a larger scale, LLaDA
shows that a masked diffusion model can be trained from scratch and
instruction-tuned as a large language model \citep{nie2025llada}, while
LLaDA~1.5 studies preference optimization for such models
\citep{zhu2025llada15}. Dream further demonstrates a strong open
diffusion LLM with parallel iterative refinement and flexible generation
orders \citep{ye2025dream}. Our work is complementary to these models and
training advances: given a base model, we study how its
conditional denoising behavior can be used at inference time.

\paragraph{Blockwise and semi-autoregressive diffusion inference time.}
Several DLMs introduce left-to-right or blockwise structure to improve
length flexibility and efficiency. SSD-LM and AR-Diffusion use
semi-autoregressive or position-dependent denoising to combine diffusion
with sequential dependencies \citep{han2023ssdlm,wu2023ardiffusion}.
Block diffusion interpolates between autoregressive and discrete
diffusion models, enabling arbitrary-length generation and KV-cache
reuse \citep{arriola2025block}. Discrete diffusion forcing similarly
turns pretrained dLLMs into an AR-diffusion hybrid for faster inference
\citep{wang2025d2f}. These methods modify the model, DLM, or
inference time schedule. In contrast, \method{} is a black-box test-time
procedure that probes the existing DLM by changing which parts of one
completed output are held fixed.

\paragraph{Adaptive unmasking and early termination.}
A closely related acceleration direction treats DLM inference as a
dynamic unmasking, token-commitment, or stopping problem. Fast-dLLM
selectively unmasks tokens whose confidence exceeds a threshold, while
using approximate KV caching to reduce per-step cost
\citep{wu2026fastdllm}. SlowFast Sampling adapts the inference time pace using
token certainty, convergence, and positional structure, alternating
between exploratory and accelerated phases \citep{wei2026slowfast}.
Learning Unmasking Policies formulates masked diffusion sampling as a
Markov decision process and learns token-unmasking decisions from model
confidences \citep{jazbec2025learning}. DAWN instead uses dependency
graphs to avoid simultaneously unmasking strongly coupled uncertain
tokens \citep{luo2026dawn}. Learn2PD trains a lightweight filter that
predicts whether each current token prediction matches the final output,
and combines this with End-of-Text Prediction to terminate inference after the
sequence is complete \citep{bao2026learning}. Prophet observes early
answer convergence and commits the remaining tokens in one step when the
top-2 confidence gap indicates sufficient stability
\citep{li2025prophet}. These methods decide when to unmask, commit, or
stop in order to reduce inference cost. \method{} has a different goal:
it starts from a completed sample and uses fresh prefix-conditioned
regenerations to test answer reproducibility and improve accuracy, rather
than shortening the original denoising run.

\paragraph{Self-consistency method.}
Self-consistency (majority-voting) has been widely used as a decoding-time strategy for improving chain-of-thought reasoning. 
Rather than relying on a single reasoning path, \citet{wang2023selfconsistency} samples multiple reasoning traces and aggregates the final answers by majority vote. 
Subsequent work has investigated how to reduce the sampling cost of this procedure through early termination. 
Adaptive Consistency \citep{aggarwal-etal-2023-lets} formulates stopping as a posterior decision problem, using a Beta-binomial model over the leading answer counts and terminating once the estimated margin of the current top answer is sufficiently large. 
Early-Stopping Self-Consistency \citep{li2024escape} instead adopts a simpler window-based criterion, stopping when all answers within a fixed-size recent window agree. 
More recently, \citet{sharma2025sequentialedge} demonstrated that sequential, entropy-aware voting can yield stronger cost-matched performance than parallel self-consistency, emphasizing the need for compute-equivalent comparisons. 
\cite{jindal2026the} proposed an inference time method that clusters short reasoning prefixes and discards prefixes on non-dominant clusters to save computation.
In a related theoretical direction, \citet{komiyama2026boinf} studied the asymptotic behavior of majority voting, or best-of-$\infty$, and proposed a Bayesian nonparametric stopping rule.

\section{Full reproduction rates Results}
\label{app:full-stability-results}

Due to space limitations, the main paper reports only the results with a keep rate of $0.5$ on GSM8K, MATH-500, and SVAMP.
Table~\ref{tab:regen-consistency-all} reports the full conditional stability diagnostics behind Figure~\ref{fig:regen-consistency-rho50}.

The full table includes CSQA and SQA across keep rates $0.1, 0.5$, and $0.9$. Across these additional datasets and rates, correct-answer retention is generally higher than same-wrong
persistence that is consistent with the results in the main paper.

\fullstabilitytable

\clearpage

\section{Full-Diffusion Accuracy Sweep}
\label{app:full-diffusion-results}

As shown in Table~\ref{tab:full-diffusion-main-results}, the proposed method consistently improves \init{} accuracy across the completed generation lengths $L=128,256,512$. This indicates that the gains reported in Table~\ref{tab:full-diffusion-main-results} are not an artifact of selecting a favorable block length, but rather reflect a robust improvement over the baseline. These results demonstrate the effectiveness of the proposed approach for improving initial-generation quality in full diffusion settings.
\begin{table}[h]
\centering
\scriptsize
\caption{
Full accuracy results for Dream and LLaDA-family diffusion language
models.
\method{} denotes majority voting over the three suffix-regenerated
completions obtained with 0.1, 0.5, and 0.9 keep rates.
Signed green/red values report accuracy-point changes relative to the
\init{} baseline. For \tif{}, the reported change is averaged over
Fixed, Linear, and Exp.\ $\alpha=5$ variants. The table includes all
completed lengths in the current full-diffusion sweep.
}
\label{tab:full-diffusion-main-results}
\resizebox{\linewidth}{!}{
\begin{tabular}{lrrrrrrrrrrrrrrr}
\toprule
\multirow{2}{*}{Method / Length} &
\multicolumn{3}{c}{GSM8K} &
\multicolumn{3}{c}{MATH-500} &
\multicolumn{3}{c}{SVAMP} &
\multicolumn{3}{c}{CSQA} &
\multicolumn{3}{c}{SQA} \\
\cmidrule(lr){2-4}\cmidrule(lr){5-7}\cmidrule(lr){8-10}\cmidrule(lr){11-13}\cmidrule(lr){14-16}
& 128 & 256 & 512 & 128 & 256 & 512 & 128 & 256 & 512 & 128 & 256 & 512 & 128 & 256 & 512 \\
\midrule

\multicolumn{16}{l}{Dream-7B} \\
\init{} baseline
& 64.06 & 80.06 & \underline{82.03}
& \underline{35.80} & 43.20 & 46.00
& 83.00 & \underline{87.00} & \underline{86.33}
& 73.22 & \underline{71.74} & 72.40
& 70.45 & 68.56 & 67.98 \\
\cmidrule(lr){1-16}
+ \tif{} Fixed
& 62.40 & 79.61 & 81.96
& 34.20 & 43.00 & 46.00
& \underline{84.00} & \underline{87.00} & \underline{86.33}
& 73.79 & 71.66 & \underline{72.48}
& \underline{70.74} & \underline{69.29} & \underline{68.56} \\
+ \tif{} Linear
& 64.52 & 80.14 & \underline{82.03}
& 34.80 & 43.40 & 46.00
& 83.67 & \underline{87.00} & \underline{86.33}
& \underline{73.96} & 71.66 & \underline{72.48}
& \underline{70.74} & \underline{69.29} & \underline{68.56} \\
+ \tif{} Exp. $\alpha=5$
& \underline{65.81} & \underline{80.36} & \underline{82.03}
& 35.60 & \underline{44.00} & \underline{46.20}
& \underline{84.00} & \underline{87.00} & \underline{86.33}
& 73.55 & \underline{71.74} & \underline{72.48}
& \underline{70.74} & \underline{69.29} & \underline{68.56} \\
$\Delta$ \tif{} avg. vs. \init{}
& \posgain{+0.18} & \neggain{-0.02} & \neggain{-0.02}
& \neggain{-0.93} & \posgain{+0.27} & \posgain{+0.07}
& \posgain{+0.89} & \posgain{+0.00} & \posgain{+0.00}
& \posgain{+0.55} & \neggain{-0.05} & \posgain{+0.08}
& \posgain{+0.29} & \posgain{+0.73} & \posgain{+0.58}\\
\cmidrule(lr){1-16}
+ \method{}
& \textbf{66.03} & \textbf{81.05} & \textbf{84.08}
& \textbf{37.20} & \textbf{44.60} & \textbf{47.00}
& \textbf{84.33} & \textbf{88.00} & \textbf{88.00}
& \textbf{74.20} & \textbf{72.24} & \textbf{73.30}
& \textbf{71.76} & \textbf{71.62} & \textbf{70.60} \\
$\Delta$ \method{} vs. \init{}
& \posgain{+1.97} & \posgain{+0.99} & \posgain{+2.05}
& \posgain{+1.40} & \posgain{+1.40} & \posgain{+1.00}
& \posgain{+1.33} & \posgain{+1.00} & \posgain{+1.67}
& \posgain{+0.98} & \posgain{+0.50} & \posgain{+0.90}
& \posgain{+1.31} & \posgain{+3.06} & \posgain{+2.62} \\

\midrule
\multicolumn{16}{l}{LLaDA-8B} \\
\init{} baseline
& \underline{58.30} & 60.73 & 18.88
& \underline{23.80} & \underline{26.45} & 17.54
& \underline{83.67} & 53.33 & 17.67
& \underline{80.51} & 42.51 & 48.32
& \underline{65.07} & 17.47 & 4.66 \\
\cmidrule(lr){1-16}
+ \tif{} Fixed
& 57.77 & 60.12 & 19.26
& \underline{24.00} & 26.25 & \underline{18.15}
& \underline{83.67} & \underline{56.33} & 18.00
& \underline{80.51} & \underline{45.29} & \textbf{51.76}
& \underline{65.07} & \textbf{28.38} & \underline{9.46} \\
+ \tif{} Linear
& 58.00 & \underline{60.80} & \underline{19.48}
& 23.80 & 26.25 & 17.94
& \underline{83.67} & \underline{56.33} & \underline{18.33}
& \underline{80.51} & 45.21 & \underline{51.60}
& \underline{65.07} & \textbf{28.38} & \underline{9.46} \\
+ \tif{} Exp. $\alpha=5$
& \underline{58.30} & \underline{60.80} & 19.41
& \underline{24.00} & 26.25 & \underline{18.15}
& \underline{83.67} & \underline{56.33} & \underline{18.33}
& \underline{80.51} & 45.21 & \underline{51.60}
& \underline{65.07} & \textbf{28.38} & \underline{9.46} \\
$\Delta$ \tif{} avg. vs. \init{}
& \neggain{-0.28} & \neggain{-0.15} & \posgain{+0.50}
& \posgain{+0.13} & \neggain{-0.20} & \posgain{+0.54}
& \posgain{+0.00} & \posgain{+3.00} & \posgain{+0.55}
& \posgain{+0.00} & \posgain{+2.73} & \posgain{+3.33}
& \posgain{+0.00} & \posgain{+10.91} & \posgain{+4.80} \\
\cmidrule(lr){1-16}
+ \method{}
& \textbf{65.88} & \textbf{65.88} & \textbf{33.36}
& \textbf{25.40} & \textbf{27.86} & \textbf{22.38}
& \textbf{84.33} & \textbf{62.67} & \textbf{31.00}
& \textbf{80.59} & \textbf{45.54} & 50.78
& \textbf{66.67} & \underline{25.62} & \textbf{9.61} \\
$\Delta$ \method{} vs. \init{}
& \posgain{+7.58} & \posgain{+5.15} & \posgain{+14.48}
& \posgain{+1.60} & \posgain{+1.41} & \posgain{+4.84}
& \posgain{+0.66} & \posgain{+9.34} & \posgain{+13.33}
& \posgain{+0.08} & \posgain{+3.03} & \posgain{+2.46}
& \posgain{+1.60} & \posgain{+8.15} & \posgain{+4.95}  \\

\midrule
\multicolumn{16}{l}{LLaDA-1.5} \\
\init{} baseline
& 56.48 & \underline{60.58} & 39.12
& \underline{23.40} & \underline{25.65} & 21.98
& \underline{83.67} & \underline{75.67} & 51.00
& \textbf{80.10} & 73.22 & 75.51
& \underline{66.08} & 51.53 & 32.02 \\
\cmidrule(lr){1-16}
+ \tif{} Fixed
& 56.41 & 60.05 & 39.80
& \underline{23.40} & \underline{25.65} & 21.98
& 83.33 & 75.00 & \underline{51.33}
& \textbf{80.10} & \textbf{74.37} & \underline{75.84}
& \underline{66.08} & \underline{56.77} & \underline{33.33} \\
+ \tif{} Linear
& \underline{56.56} & 60.27 & 39.95
& \underline{23.40} & 25.45 & \underline{22.38}
& \underline{83.67} & 75.33 & \underline{51.33}
& \textbf{80.10} & \underline{74.28} & 75.76
& \underline{66.08} & \textbf{57.06} & 33.19 \\
+ \tif{} Exp. $\alpha=5$
& \underline{56.56} & 60.50 & \underline{40.03}
& \underline{23.40} & \underline{25.65} & \underline{22.38}
& \underline{83.67} & \underline{75.67} & \underline{51.33}
& \textbf{80.10} & \underline{74.28} & 75.76
& \underline{66.08} & \textbf{57.06} & 33.19 \\
$\Delta$ \tif{} avg. vs. \init{}
& \posgain{+0.03} & \neggain{-0.30} & \posgain{+0.81}
& \posgain{+0.00} & \neggain{-0.07} & \posgain{+0.27}
& \neggain{-0.11} & \neggain{-0.34} & \posgain{+0.33}
& \posgain{+0.00} & \posgain{+1.09} & \posgain{+0.28}
& \posgain{+0.00} & \posgain{+5.43} & \posgain{+1.22} \\
\cmidrule(lr){1-16}
+ \method{}
& \textbf{65.88} & \textbf{66.26} & \textbf{56.79}
& \textbf{26.40} & \textbf{26.65} & \textbf{26.01}
& \textbf{84.67} & \textbf{79.00} & \textbf{62.67}
& \textbf{80.10} & \underline{74.28} & \textbf{76.09}
& \textbf{66.81} & 52.98 & \textbf{40.32} \\
$\Delta$ \method{} vs. \init{}
& \posgain{+9.40} & \posgain{+5.68} & \posgain{+17.67}
& \posgain{+3.00} & \posgain{+1.00} & \posgain{+4.03}
& \posgain{+1.00} & \posgain{+3.33} & \posgain{+11.67}
& \posgain{+0.00} & \posgain{+1.06} & \posgain{+0.58}
& \posgain{+0.73} & \posgain{+1.45} & \posgain{+8.30} \\

\bottomrule
\end{tabular}
}
\end{table}

\section{Semi-AR LLaDA-Family Results}
\label{app:semi-ar-llada}

Table~\ref{tab:semi-ar-llada-consistency} and Table~\ref{tab:semi-ar-llada-main-results} report the LLaDA-family results with block
length 32.\footnote{Note that all results reported in the main paper correspond to the setting in which the block length equals the generation length, referred to as the full diffusion setting.} 
When the block length is smaller than the generation length, the sequence is generated block by block, with diffusion performed independently within each block, referred to as the semi-autoregressive blockwise setting.

Overall, the results in Table~\ref{tab:semi-ar-llada-consistency} and Table~\ref{tab:semi-ar-llada-main-results} show that the proposed method remains effective even in this semi-autoregressive blockwise setting. Across generation lengths, the proposed method all-in-all improves \init{} accuracy for the LLaDA-family models, indicating that its benefits are not limited to the full-diffusion setting used for the main claims.

\begin{table*}[h]
  \centering
  \scriptsize
  \caption{
  Semi-AR LLaDA-family prefix-denoising consistency results. These runs
  use block length 32 with generation lengths 128, 256, and 512.
  Values are percentages. Bold marks the largest $\Delta$ within each
  row; negative $\Delta$ values are shown in red.
}
  \label{tab:semi-ar-llada-consistency}

  \footnotesize
  \setlength{\tabcolsep}{2.7pt}
  \renewcommand{\arraystretch}{1.08}

  \sisetup{
    table-number-alignment=center,
    detect-weight=true,
    detect-inline-weight=math
  }

  \begin{threeparttable}
  \begin{adjustbox}{max width=\textwidth}
  \begin{tabular}{@{}llc
                  S[table-format=3.1]
                  S[table-format=3.1]
                  S[table-format=3.1]
                  @{\hspace{6pt}}
                  S[table-format=3.1]
                  S[table-format=3.1]
                  S[table-format=3.1]
                  @{\hspace{6pt}}
                  rrr
                  @{}}
    \toprule
    \multirow{2}{*}{\textbf{Dataset}}
    & \multirow{2}{*}{\textbf{Model}}
    & \multirow{2}{*}{\textbf{Length}}
    & \multicolumn{3}{c}{$R^+_\rho$}
    & \multicolumn{3}{c}{$R^-_\rho$}
    & \multicolumn{3}{c}{$\Delta$} \\
    \cmidrule(lr){4-6}
    \cmidrule(lr){7-9}
    \cmidrule(l){10-12}
    & &
    & {0.1} & {0.5} & {0.9}
    & {0.1} & {0.5} & {0.9}
    & {0.1} & {0.5} & {0.9} \\
    \midrule

    \multirow{6}{*}{GSM8K}
    & \multirow{3}{*}{LLaDA-1.5}
    & 128 & 95.6 & 99.5 & 99.0 & 78.0 & 91.5 & 89.7 & \textbf{17.6} & 8.0 & 9.2 \\
    & & 256 & 96.1 & 98.7 & 99.7 & 72.8 & 90.6 & 96.7 & \textbf{23.3} & 8.2 & 3.0 \\
    & & 512 & 97.7 & 99.2 & 99.9 & 67.8 & 86.8 & 98.7 & \textbf{29.9} & 12.4 & 1.2 \\
    \cmidrule(lr){2-12}
    & \multirow{3}{*}{LLaDA-8B}
    & 128 & 95.7 & 98.3 & 99.4 & 74.9 & 87.2 & 91.7 & \textbf{20.8} & 11.2 & 7.6 \\
    & & 256 & 97.3 & 99.4 & 99.6 & 70.1 & 87.8 & 95.6 & \textbf{27.2} & 11.6 & 4.0 \\
    & & 512 & 98.1 & 99.5 & 99.9 & 67.1 & 82.3 & 95.4 & \textbf{31.0} & 17.3 & 4.6 \\

    \midrule

    \multirow{6}{*}{MATH-500}
    & \multirow{3}{*}{LLaDA-1.5}
    & 128 & 89.4 & 95.9 & 97.6 & 63.0 & 85.3 & 90.9 & \textbf{26.3} & 10.6 & 6.7 \\
    & & 256 & 88.6 & 95.1 & 99.5 & 54.3 & 80.7 & 94.2 & \textbf{34.3} & 14.4 & 5.3 \\
    & & 512 & 88.5 & 93.8 & 99.5 & 45.6 & 66.6 & 93.4 & \textbf{42.8} & 27.2 & 6.1 \\
    \cmidrule(lr){2-12}
    & \multirow{3}{*}{LLaDA-8B}
    & 128 & 88.8 & 97.5 & 97.5 & 59.4 & 84.0 & 89.9 & \textbf{29.4} & 13.5 & 7.7 \\
    & & 256 & 90.3 & 96.6 & 99.4 & 55.0 & 77.7 & 94.2 & \textbf{35.3} & 18.9 & 5.3 \\
    & & 512 & 82.7 & 91.1 & 98.6 & 43.1 & 70.7 & 93.8 & \textbf{39.6} & 20.5 & 4.8 \\

    \midrule

    \multirow{6}{*}{SVAMP}
    & \multirow{3}{*}{LLaDA-1.5}
    & 128 & 98.9 & 99.6 & 99.6 & 68.6 & 88.6 & 100.0 & \textbf{30.3} & 11.1 & -0.4\\
    & & 256 & 98.5 & 99.6 & 100.0 & 70.0 & 87.5 & 97.5 & \textbf{28.5} & 12.1 & 2.5 \\
    & & 512 & 98.1 & 99.3 & 100.0 & 74.3 & 91.4 & 97.1 & \textbf{23.8} & 7.8 & 2.9 \\
    \cmidrule(lr){2-12}
    & \multirow{3}{*}{LLaDA-8B}
    & 128 & 98.5 & 99.6 & 99.2 & 74.3 & 88.6 & 97.1 & \textbf{24.2} & 11.0 & 2.1 \\
    & & 256 & 98.9 & 100.0 & 100.0 & 60.5 & 84.2 & 94.7 & \textbf{38.3} & 15.8 & 5.3 \\
    & & 512 & 98.5 & 99.3 & 100.0 & 73.5 & 91.2 & 97.1 & \textbf{25.0} & 8.1 & 2.9 \\

    \midrule

    \multirow{6}{*}{CSQA}
    & \multirow{3}{*}{LLaDA-1.5}
    & 128 & 95.7 & 98.3 & 99.3 & 74.5 & 90.7 & 95.7 & \textbf{21.3} & 7.6 & 3.6 \\
    & & 256 & 93.8 & 97.8 & 99.4 & 71.6 & 89.4 & 98.6 & \textbf{22.2} & 8.4 & 0.8 \\
    & & 512 & 94.3 & 97.6 & 99.6 & 69.0 & 84.5 & 97.5 & \textbf{25.3} & 13.1 & 2.0 \\
    \cmidrule(lr){2-12}
    & \multirow{3}{*}{LLaDA-8B}
    & 128 & 94.8 & 98.7 & 98.8 & 74.1 & 90.7 & 96.4 & \textbf{20.7} & 8.1 & 2.4 \\
    & & 256 & 92.3 & 97.3 & 99.2 & 70.6 & 88.6 & 97.8 & \textbf{21.8} & 8.7 & 1.4 \\
    & & 512 & 93.9 & 96.8 & 99.2 & 65.0 & 84.3 & 96.7 & \textbf{29.0} & 12.5 & 2.4 \\
    \bottomrule
  \end{tabular}
  \end{adjustbox}
  \end{threeparttable}
\end{table*}

\begin{table}[t]
\centering
\scriptsize
\caption{
Semi-AR LLaDA-family main accuracy results. These runs use block length
32 with generation lengths 128, 256, and 512.
}
\label{tab:semi-ar-llada-main-results}
\resizebox{\linewidth}{!}{
\begin{tabular}{lrrrrrrrrrrrr}
\toprule
\multirow{2}{*}{Method / Length} &
\multicolumn{3}{c}{GSM8K} &
\multicolumn{3}{c}{MATH-500} &
\multicolumn{3}{c}{SVAMP} &
\multicolumn{3}{c}{CSQA} \\
\cmidrule(lr){2-4}\cmidrule(lr){5-7}\cmidrule(lr){8-10}\cmidrule(lr){11-13}
& 128 & 256 & 512 & 128 & 256 & 512 & 128 & 256 & 512 & 128 & 256 & 512 \\
\midrule

\multicolumn{13}{l}{LLaDA-8B} \\
\init{} baseline
& 72.25 & 74.91 & 81.80
& \underline{32.20} & \underline{35.27} & \textbf{43.15}
& 86.00 & \underline{87.00} & \textbf{88.67}
& 76.58 & 76.99 & 77.07\\
\cmidrule(lr){1-13}
+ \tif{} Fixed
& 71.87 & 69.37 & 80.14
& 29.60 & 31.06 & 39.31
& 84.67 & 86.67 & \underline{88.33}
& \textbf{78.79} & \underline{77.07} & \textbf{77.40} \\
+ \tif{} Linear
& 73.54 & 74.75 & 81.65
& 30.80 & 34.67 & 41.53
& \underline{87.67} & 86.33 & \underline{88.33}
& \underline{77.31} & \underline{77.07} & \textbf{77.40} \\
+ \tif{} Exp. $\alpha=5$
& \underline{73.62} & \textbf{76.27} & \underline{81.96}
& 31.80 & \underline{35.27} & \underline{42.54}
& \textbf{88.00} & 86.67 & \underline{88.33}
& 76.99 & \textbf{77.31} & \textbf{77.40}\\
$\Delta$ \tif{} avg. vs. \init{}
& \posgain{+0.76} & \neggain{-1.45} & \neggain{-0.55}
& \neggain{-1.47} & \neggain{-1.60} & \neggain{-2.02}
& \posgain{+0.78} & \neggain{-0.44} & \neggain{-0.34}
& \posgain{+1.12} & \posgain{+0.16} & \posgain{+0.33} \\
\cmidrule(lr){1-13}
+ \method{}
& \textbf{73.77} & \underline{76.04} & \textbf{82.49}
& \textbf{33.20} & \textbf{36.47} & 42.34
& 87.00 & \textbf{88.00} & \underline{88.33}
& \underline{77.31} & \textbf{77.31} & \textbf{77.40} \\
$\Delta$ \method{} vs. \init{}
& \posgain{+1.52} & \posgain{+1.13} & \posgain{+0.69}
& \posgain{+1.00} & \posgain{+1.20} & \neggain{-0.81}
& \posgain{+1.00} & \posgain{+1.00} & \neggain{-0.34}
& \posgain{+0.73} & \posgain{+0.32} & \posgain{+0.33}  \\

\midrule
\multicolumn{13}{l}{LLaDA-1.5} \\
\init{} baseline
& 72.33 & 78.01 & \underline{82.71}
& \underline{34.00} & 36.87 & \textbf{41.94}
& 87.67 & \underline{86.67} & \underline{88.33}
& 76.66 & 76.58 & 75.92 \\
\cmidrule(lr){1-13}
+ \tif{} Fixed
& 71.95 & 72.48 & 80.74
& 32.80 & 33.27 & 38.91
& 84.33 & 86.33 & 87.33
& \textbf{78.30} & \underline{77.23} & 75.92 \\
+ \tif{} Linear
& 73.31 & 77.56 & 82.56
& 33.40 & 36.67 & 41.13
& 87.67 & 86.33 & 87.67
& \underline{77.31} & 77.15 & 76.09 \\
+ \tif{} Exp. $\alpha=5$
& \underline{73.46} & \textbf{79.00} & \textbf{83.02}
& 33.80 & \underline{37.07} & \textbf{41.94}
& \textbf{88.33} & 86.33 & \underline{88.33}
& 76.82 & 76.58 & \underline{76.17} \\
$\Delta$ \tif{} avg. vs. \init{}
& \posgain{+0.58} & \neggain{-1.66} & \neggain{-0.60}
& \neggain{-0.67} & \neggain{-1.20} & \neggain{-1.28}
& \neggain{-0.89} & \neggain{-0.34} & \neggain{-0.55}
& \posgain{+0.82} & \posgain{+0.41} & \posgain{+0.14} \\
\cmidrule(lr){1-13}
+ \method{}
& \textbf{73.92} & \underline{78.47} & \underline{82.71}
& \textbf{34.41} & \textbf{37.27} & \underline{41.53}
& \underline{88.00} & \textbf{87.67} & \textbf{89.00}
& 77.23 & \textbf{77.31} & \textbf{77.64}\\
$\Delta$ \method{} vs. \init{}
& \posgain{+1.59} & \posgain{+0.46} & \posgain{+0.00}
& \posgain{+0.41} & \posgain{+0.40} & \neggain{-0.41}
& \posgain{+0.33} & \posgain{+1.00} & \posgain{+0.67}
& \posgain{+0.57} & \posgain{+0.73} & \posgain{+0.72} \\

\bottomrule
\end{tabular}
}
\end{table}

\section{Case Study: How Prefix Regeneration Repairs Errors}
\label{sec:theoretical-view}

\begin{table*}[t]
\centering
\small
\begin{tcolorbox}[
  colback=black!1,
  colframe=black!25,
  title={Problem 97},
  fonttitle=\bfseries,
  boxrule=0.35pt,
  arc=1pt,
  left=5pt,
  right=5pt,
  top=5pt,
  bottom=5pt]
\textbf{Problem.}
\emph{A figure skater is facing north when she begins to spin to her right. She
spins 2250 degrees. Which direction (north, south, east or west) is she facing
when she finishes her spin?}

\vspace{0.5em}
\noindent
\begin{minipage}[t]{0.485\linewidth}
\textbf{\textcolor{red!70!black}{Initial generation}}
\vspace{0.25em}
\begin{tcolorbox}[
  colback=red!2,
  colframe=red!35,
  boxrule=0.3pt,
  arc=1pt,
  left=4pt,
  right=4pt,
  top=4pt,
  bottom=4pt]
The figure skater is facing north when she begins to spin to her right. She
spins 2250 degrees.

\smallskip
\ldots{} We can do this by subtracting multiples of 360
degrees from 2250 degrees until we get an angle between 0 and 360 degrees.

\smallskip
\textcolor{red!70!black}{2250 degrees divided by 360 degrees is 6 with a
remainder of 150 degrees.}

\smallskip
\textcolor{red!70!black}{Therefore, she is facing $\backslash\mathrm{boxed}\{\mathrm{south}\}$.}
\end{tcolorbox}
\end{minipage}\hfill
\begin{minipage}[t]{0.485\linewidth}
\textbf{\textcolor{green!45!black}{Prefix-conditioned regeneration}}
\vspace{0.25em}
\begin{tcolorbox}[
  colback=green!2,
  colframe=green!35!black,
  boxrule=0.3pt,
  arc=1pt,
  left=4pt,
  right=4pt,
  top=4pt,
  bottom=4pt]
The figure skater is facing north when she begins to spin to her right. She
spins 2250 degrees.

\smallskip
\ldots{} We can do this by subtracting multiples of 360
degrees from 2250 degrees.

\smallskip
\textcolor{green!45!black}{2250 - 6*360 = 2250 - 2160 = 90 degrees.}

\smallskip
\textcolor{green!45!black}{So, when she finishes her spin, she is facing
$\backslash\mathrm{boxed}\{\text{east}\}$.}
\end{tcolorbox}
\end{minipage}
\end{tcolorbox}

\vspace{0.6em}

\begin{tcolorbox}[
  colback=black!1,
  colframe=black!25,
  title={Problem 144},
  fonttitle=\bfseries,
  boxrule=0.35pt,
  arc=1pt,
  left=5pt,
  right=5pt,
  top=5pt,
  bottom=5pt]
\textbf{Problem.}
\emph{A curve is parameterized by
\[
(x,y) = (t^3 + 7, -3t^2 - 6t - 5).
\]
Find the point the curve passes through at $t = 2$.}

\vspace{0.5em}
\noindent
\begin{minipage}[t]{0.485\linewidth}
\textbf{\textcolor{red!70!black}{Initial generation}}
\vspace{0.25em}
\begin{tcolorbox}[
  colback=red!2,
  colframe=red!35,
  boxrule=0.3pt,
  arc=1pt,
  left=4pt,
  right=4pt,
  top=4pt,
  bottom=4pt]
When $t = 2,$ we have
\[
\begin{aligned}
(x,y) &= (2^3 + 7, -3(2)^2 - 6(2) - 5) \\
&= \textcolor{red!70!black}{(15, -27)}.
\end{aligned}
\]
\textcolor{red!70!black}{So the curve passes through the point
$\backslash\mathrm{boxed}\{(15,-27)\}.$}
\end{tcolorbox}
\end{minipage}\hfill
\begin{minipage}[t]{0.485\linewidth}
\textbf{\textcolor{green!45!black}{Prefix-conditioned regeneration}}
\vspace{0.25em}
\begin{tcolorbox}[
  colback=green!2,
  colframe=green!35!black,
  boxrule=0.3pt,
  arc=1pt,
  left=4pt,
  right=4pt,
  top=4pt,
  bottom=4pt]
When $t = 2,$ we have
\[
\begin{aligned}
(x,y) &= (2^3 + 7, -3(2)^2 - 6(2) - 5) \\
&= \textcolor{green!45!black}{(8 + 7, -12 - 12 - 5)} \\
&= \textcolor{green!45!black}{(15, -29)}.
\end{aligned}
\]
\textcolor{green!45!black}{Therefore, the point the curve passes through at
$t = 2$ is $\backslash\mathrm{boxed}\{(15, -29)\}.$ The answer is: (15,-29)}
\end{tcolorbox}
\end{minipage}
\end{tcolorbox}
\caption{
Two examples in which the wrong answer is corrected by regeneration.
Red text highlights the incorrect continuation in the initial sample; green text
highlights the corrected continuation produced by regeneration. 
When the wrong intermediate step is remasked, prefix-conditioned regeneration avoids the same mistake.
}
\label{tab:regen-correction-examples}
\end{table*}
Section~\ref{sec:pdc} shows that prefix-conditioned regeneration has high correct-answer preservation
($R^+$) and low same-wrong preservation ($R^-$). In words, correct initial
answers tend to remain correct after regeneration, while wrong initial answers
often do not reproduce the same wrong answer.
This is a useful regime for \method{}. High $R^+$ means regeneration is
unlikely to disturb a correct trajectory, while a low $R^-$ means an incorrect
trajectory often has a chance to escape its original mistake. Escaping the
same wrong answer does not by itself guarantee correctness; the regenerated
answer may still be wrong. However, when the kept prefix contains the right
setup and the erroneous intermediate step lies in the regenerated suffix,
regeneration can cut away the local mistake and replace it with a corrected
continuation. 
Table~\ref{tab:regen-correction-examples} shows two MATH-500 with Dream-v0-Instruct-7B
examples of this behavior.

In Problem 97, the initial sample follows the correct modulo-reduction
strategy but makes a local arithmetic error, treating the remainder of
$2250$ degrees modulo $360$ degrees as $150$ degrees and concluding
\emph{south}. The regenerated sample recomputes the suffix as
$2250 - 6\cdot 360 = 90$ degrees and recovers \emph{east}. Similarly, in
Problem 144, the initial sample substitutes the correct value $t=2$, but
evaluates the resulting expression as $-27$, whereas regeneration expands the
same expression as $-12 - 12 - 5 = -29$ and recovers the correct point.

\section{Model and Dataset Summary}
\label{app:model-dataset-summary}

This appendix summarizes the language models and datasets used in the
experiments. Detailed inference time hyperparameters, prompt templates, answer
extraction, and voting conventions are given in
Appendix~\ref{app:reproducibility-details}.

\begin{table}[t]
\centering
\caption{
Language models used in the experiments. Dream-7B is exclusively for full diffusion, while LLaDA-family models support both full diffusion and semi-autoregressive diffusion.
The main paper exclusively considers the full diffusion setting. The semi-autoregressive 
runs are reported separately in Appendix~\ref{app:semi-ar-llada}.
}
\label{tab:model-summary}
\small
\begin{adjustbox}{max width=\textwidth}
\begin{tabular}{llll}
\toprule
\textbf{Model} & \textbf{Checkpoint} & \textbf{Family} & \textbf{Main setting} \\
\midrule
Dream-7B
& \texttt{Dream-org/Dream-v0-Instruct-7B}
& Masked Diffusion LLM
& Full-window diffusion inference time \\
LLaDA-1.5
& \texttt{GSAI-ML/LLaDA-1.5}
& Masked Diffusion LLM with Block Length
& Block length matched to generation length \\
LLaDA-8B
& \texttt{GSAI-ML/LLaDA-8B-Instruct}
& Masked Diffusion LLM with Block Length
& Block length matched to generation length \\
\bottomrule
\end{tabular}
\end{adjustbox}
\end{table}

\begin{table}[t]
\centering
\caption{
Datasets used in the experiments. Math-style datasets are evaluated by
canonical answer extraction and symbolic/numeric equivalence when
applicable. Multiple-choice and binary reasoning datasets are evaluated
by canonicalized option extraction.
}
\label{tab:dataset-summary}
\small
\begin{adjustbox}{max width=\textwidth}
\begin{tabular}{llll}
\toprule
\textbf{Dataset} & \textbf{Name} & \textbf{Task type} & \textbf{Primary answer format} \\
\midrule
GSM8K
& Grade-school math word problems \citep{cobbe2021gsm8k}
& Arithmetic reasoning
& Boxed final answer \\
MATH-500
& MATH subset / verification benchmark \citep{hendrycks2021math,lightman2024verify}
& Mathematical reasoning
& Boxed final answer \\
SVAMP
& Arithmetic word-problem challenge \citep{patel2021svamp}
& Arithmetic reasoning
& Boxed final answer \\
CSQA
& CommonsenseQA \citep{talmor-etal-2019-commonsenseqa}
& Multiple-choice commonsense reasoning
& Boxed option letter \\
SQA
& StrategyQA \citep{geva2021strategyqa}
& Binary commonsense reasoning
& Boxed option letter \\
\bottomrule
\end{tabular}
\end{adjustbox}
\end{table}

\section{Results at Different Sampling Temperatures}\label{app:high-temp}

To examine whether \method{} remains effective as sampling becomes more
stochastic, we repeat the $L=128$ evaluation at token-sampling
temperature $\tau= 1.0$. 

\begin{table*}[t]
\centering
\caption{
Accuracy at generation length $L=128$ under low- and high-temperature
decoding. We compare the main setting ($\tau=0.2$) with more stochastic
decoding ($\tau=1.0$).
\init{}~$\times 4$ denotes majority voting over four independent full
generations. Boldface marks the best result in each model--temperature
row. Signed green/red values report the accuracy-point change of
\method{} relative to \init{}~$\times 4$ at the same temperature.
}
\label{tab:T1-0}
\scriptsize
\setlength{\tabcolsep}{3.5pt}
\begin{adjustbox}{max width=\textwidth}
\begin{tabular}{lllrrrrrrr}
\toprule
\textbf{Dataset}
& \textbf{Model}
& $\boldsymbol{\tau}$
& \init{}
& \init{}~$\times 4$
& \tif{} Fixed~$\times 4$
& \tif{} Linear~$\times 4$
& \tif{} Exp.~$\times 4$
& \method{}
& $\Delta$ vs. \init{}~$\times 4$ \\
\midrule
\multirow{6}{*}{GSM8K}
& \multirow{2}{*}{Dream-7B}
& 0.2 & 64.06 & 64.22 & 62.40 & 64.52 & 65.81 & \textbf{66.03} & \posgain{+1.81} \\
& & 1.0 & 41.24 & 50.19 & 45.11 & 47.08 & 49.13 & \textbf{65.96} & \posgain{+15.77} \\
\cmidrule(lr){2-10}
& \multirow{2}{*}{LLaDA-1.5}
& 0.2 & 56.48 & 56.63 & 56.41 & 56.56 & 56.71 & \textbf{65.88} & \posgain{+9.25} \\
& & 1.0 & 56.94 & 58.83 & 58.15 & 58.15 & 58.23 & \textbf{68.69} & \posgain{+9.86} \\
\cmidrule(lr){2-10}
& \multirow{2}{*}{LLaDA-8B}
& 0.2 & 58.30 & 58.91 & 58.38 & 58.61 & 58.91 & \textbf{65.88} & \posgain{+6.97} \\
& & 1.0 & 57.54 & 60.73 & 59.14 & 59.74 & 60.05 & \textbf{69.45} & \posgain{+8.72} \\
\midrule
\multirow{6}{*}{MATH-500}
& \multirow{2}{*}{Dream-7B}
& 0.2 & 35.80 & 36.00 & 34.20 & 34.80 & 35.60 & \textbf{37.20} & \posgain{+1.20} \\
& & 1.0 & 22.40 & 25.20 & 24.20 & 24.60 & 25.00 & \textbf{37.60} & \posgain{+12.40} \\
\cmidrule(lr){2-10}
& \multirow{2}{*}{LLaDA-1.5}
& 0.2 & 23.40 & 23.20 & 23.40 & 23.40 & 23.00 & \textbf{26.40} & \posgain{+3.20} \\
& & 1.0 & 24.00 & 23.20 & 23.60 & 23.60 & 23.40 & \textbf{30.20} & \posgain{+7.00} \\
\cmidrule(lr){2-10}
& \multirow{2}{*}{LLaDA-8B}
& 0.2 & 23.80 & 23.60 & 23.80 & 23.80 & 23.60 & \textbf{25.40} & \posgain{+1.80} \\
& & 1.0 & 24.45 & 25.05 & 23.85 & 23.85 & 24.25 & \textbf{28.06} & \posgain{+3.01} \\
\midrule
\multirow{6}{*}{SVAMP}
& \multirow{2}{*}{Dream-7B}
& 0.2 & 83.00 & 83.00 & 84.00 & 83.67 & 84.00 & \textbf{84.33} & \posgain{+1.33} \\
& & 1.0 & 70.33 & 75.67 & 75.00 & 75.67 & 76.00 & \textbf{82.67} & \posgain{+7.00} \\
\cmidrule(lr){2-10}
& \multirow{2}{*}{LLaDA-1.5}
& 0.2 & 83.67 & 83.67 & 83.33 & 83.67 & 83.67 & \textbf{84.67} & \posgain{+1.00} \\
& & 1.0 & 85.00 & 84.67 & 85.00 & 84.67 & 84.67 & \textbf{89.00} & \posgain{+4.33} \\
\cmidrule(lr){2-10}
& \multirow{2}{*}{LLaDA-8B}
& 0.2 & 83.67 & \textbf{84.33} & \textbf{84.33} & 84.00 & \textbf{84.33} & \textbf{84.33} & \posgain{+0.00} \\
& & 1.0 & 84.00 & 84.00 & 85.00 & 84.33 & 84.33 & \textbf{86.33} & \posgain{+2.33} \\
\midrule
\multirow{6}{*}{CSQA}
& \multirow{2}{*}{Dream-7B}
& 0.2 & 73.22 & 73.22 & 73.79 & 73.96 & 73.55 & \textbf{74.20} & \posgain{+0.98} \\
& & 1.0 & 64.78 & 76.49 & \textbf{78.05} & 77.89 & 77.81 & 74.37 & \neggain{-2.12} \\
\cmidrule(lr){2-10}
& \multirow{2}{*}{LLaDA-1.5}
& 0.2 & 80.10 & 80.51 & \textbf{80.67} & \textbf{80.67} & \textbf{80.67} & 80.10 & \neggain{-0.41} \\
& & 1.0 & 80.10 & \textbf{80.84} & 80.67 & 80.75 & 80.75 & 79.44 & \neggain{-1.40} \\
\cmidrule(lr){2-10}
& \multirow{2}{*}{LLaDA-8B}
& 0.2 & 80.51 & 80.34 & \textbf{80.59} & \textbf{80.59} & \textbf{80.59} & \textbf{80.59} & \posgain{+0.25} \\
& & 1.0 & 79.69 & 80.59 & 80.75 & \textbf{80.92} & 80.67 & 79.85 & \neggain{-0.74} \\
\midrule
\multirow{6}{*}{SQA}
& \multirow{2}{*}{Dream-7B}
& 0.2 & 70.45 & 70.45 & 70.74 & 70.74 & 70.74 & \textbf{71.76} & \posgain{+1.31} \\
& & 1.0 & 58.66 & 63.17 & 64.48 & 64.77 & 64.77 & \textbf{67.10} & \posgain{+3.93} \\
\cmidrule(lr){2-10}
& \multirow{2}{*}{LLaDA-1.5}
& 0.2 & 66.08 & 65.94 & 65.94 & 65.79 & 65.79 & \textbf{66.81} & \posgain{+0.87} \\
& & 1.0 & 65.07 & 64.77 & 64.92 & 65.21 & 64.92 & \textbf{66.96} & \posgain{+2.19} \\
\cmidrule(lr){2-10}
& \multirow{2}{*}{LLaDA-8B}
& 0.2 & 65.07 & 64.63 & 65.07 & 65.07 & 65.07 & \textbf{66.67} & \posgain{+2.04} \\
& & 1.0 & 64.77 & 64.19 & 64.63 & 64.63 & 64.92 & \textbf{66.67} & \posgain{+2.48} \\
\bottomrule
\end{tabular}
\end{adjustbox}
\end{table*}

\paragraph{Results.}
Table~\ref{tab:T1-0} shows that, at $\tau=1.0$, \method{} achieves the
highest numerical accuracy in 12
of the 15 model and dataset settings and outperforms \init{}~$\times 4$ in
the same 12 settings. Averaged uniformly over the 15 settings, it exceeds
\init{}~$\times 4$ by $4.98$ accuracy points; the largest dataset-level
mean gains occur on GSM8K and MATH-500, at $11.45$ and $7.47$ points,
respectively. Across the two tested temperatures, the unweighted mean
accuracy of \method{} changes from $65.35$ at $\tau=0.2$ to $66.16$ at
$\tau=1.0$ ($+0.81$), whereas \init{} and \init{}~$\times 4$ decrease by
$4.58$ and $2.07$ points, respectively. Consequently, the mean margin of
\method{} over \init{}~$\times 4$ increases from $2.11$ to $4.98$ points.
This larger relative margin partly reflects a weaker full-generation
baselines at the higher temperature and are not uniform across tasks:
on all three CSQA settings, \method{} trails \init{}~$\times 4$ by
$0.74$--$2.12$ points. The results therefore show that \method{} remains
effective in the tested higher-temperature setting, particularly on the
mathematical reasoning tasks.

\section{Reproducibility Details}
\label{app:reproducibility-details}

\paragraph{Models and inference time hyperparameters.}
We evaluate Dream-v0-Instruct-7B, LLaDA-1.5, and LLaDA-8B-Instruct.
We use entropy-based and low confidence unmasking, token temperature $0.2$. The generation
length $L$ is set equal to the denoising step budget $T\in\{128,256,512\}$.

The main tables use the block-length-matched setting, with block length set
equal to the generation length. Appendix~\ref{app:semi-ar-llada} reports semi-AR LLaDA-family runs with block length 32, as follows the \tif{} settings. \tif{} are captured from the same initial denoising run with a stride of 1.

\paragraph{Regeneration denoising budget.}
During regeneration, the effective denoising budget is
reduced according to the kept prefix length. In the reported full-diffusion setting, $T=L$, so the effective regeneration step budget equals the remasked suffix length:
\[
    T_\rho
    =
    L_\rho
    =
    L-L_\rho^{\text{prefix}}.
\]
The experiment scripts also support a general nominal step budget $T$,
using the proportional rule:
\[
    T_\rho
    =
    \max\{1,\operatorname{round}(T L_\rho/L)\}.
\]
This reduces to $T_\rho=L_\rho$ when $T=L$. Dream regenerates the suffix with a
reduced number of denoising steps using the floor version of the same
proportional rule:
\[
    T_\rho
    =
    \max\{1,\lfloor T L_\rho/L \rfloor\}.
\]
This differs from the rounded rule by at most one denoising step, and all
reported tables are computed from the saved records for the corresponding
run.

\paragraph{Prompt template.}
For math-style tasks, all prompts use the same instruction suffix:
\begin{quote}
\small
\texttt{Please reason step by step, and put your final answer within
\textbackslash boxed\{\}.}
\end{quote}
For multiple-choice tasks, we explicitly require a chain-of-thought style
reasoning before the final choice. The prompt lists the answer options
and includes the following mandatory format instruction:
\begin{quote}
\small
\texttt{Mandatory format (do not skip):}\\
\texttt{1) Write your reasoning across several sentences (facts,
definitions, or everyday knowledge you use).}\\
\texttt{2) Only after that reasoning, end with exactly one line that
contains only one capital option letter (A--E) inside
\textbackslash boxed\{\}, e.g. \textbackslash boxed\{C\}.}
\end{quote}
Thus, the multiple-choice setting is not a direct-answer-only setting;
the model is instructed to produce reasoning and then end with a single
boxed option letter.

\paragraph{Answer extraction and equivalence.}
For math datasets, we extract the last boxed expression
when present. If no boxed expression is found, it applies final-answer
fallback patterns such as ``The answer is'', ``answer:'', and ``final
answer is'' in the last lines of the output. Extracted answers are
normalized using Hendrycks-MATH-style string normalization, including
removal of superficial LaTeX wrappers, unit text, spacing, percent signs,
and simple fraction normalization. If normalized strings do not match
exactly, we use a SymPy fallback: both expressions are parsed with the
LaTeX parser when possible, otherwise with sympify, and are
judged equivalent when the symbolic simplification of their difference is
zero. Numeric decimal strings are compared with tolerance $10^{-6}$.

For multiple-choice tasks, the extracted answer is canonicalized to a
choice letter. Exact option-text matches are mapped to their corresponding
letter before loose letter-form matching, to avoid treating option text
such as formulas as a choice letter. Common letter forms such as
\texttt{C}, \texttt{(C)}, \texttt{C.}, and
\texttt{\textbackslash text\{C\}} are accepted.

\paragraph{Null predictions, no-vote cases, and missing triplets.}
If no answer can be extracted from an initial output, the example is
counted as incorrect for \init{} accuracy. Conditional quantities that
require a non-null initial answer, such as same-wrong persistence and the
probability that an initially wrong answer becomes correct, exclude such
examples from the conditioning set.

\end{document}